%% file: main.tex
\documentclass[12pt]{article}
\usepackage{comment,etex,authblk,amsmath,fullpage,amsfonts,titletoc,titlesec,amssymb}
\usepackage[onehalfspacing]{setspace}
\usepackage{array,multirow,placeins}
\usepackage{subcaption}
\usepackage{natbib}
\usepackage{booktabs}
\usepackage{caption}
\usepackage{amssymb}
\usepackage{mathtools}
\usepackage{amsthm}
\usepackage{multirow}
\usepackage{float}
\usepackage{algorithm}
\usepackage{algpseudocode}
\usepackage{authblk}
\usepackage{titlesec}

\newtheorem{lemma}{Lemma}

\newcommand{\gl}{\lambda}

\newcommand{\gd}{\delta}

\newcommand{\R}{{\mathbb R}}

\newcommand{\cN}{\mathcal N}

\newcommand{\bx}{{\bf x}}
\newcommand{\vertiii}[1]{{\left\vert\kern-0.25ex\left\vert\kern-0.25ex\left\vert #1 
\right\vert\kern-0.25ex\right\vert\kern-0.25ex\right\vert}}

\usepackage{hyperref}

\author[$\dagger$]{Teng Andrea Xu}
\author[$\ddagger$]{Bryan Kelly}
\author[$\dagger$]{Semyon Malamud}

\affil[$\dagger$]{Swiss Finance Institute, EPFL \authorcr \texttt{andrea.xu,semyon.malamud@epfl.ch}}
\affil[$\ddagger$]{Yale School of Management, Yale University \authorcr \texttt{bryan.kelly@yale.edu }} 
\begin{document}

\title{A Simple Algorithm For Scaling Up Kernel Methods}

%

%
\date{}

\maketitle

\begin{abstract}
 The recent discovery of the equivalence between infinitely wide neural networks (NNs) in the lazy training regime and Neural Tangent Kernels (NTKs) \cite{jacot2018neural} has revived interest in kernel methods. However, conventional wisdom suggests kernel methods are unsuitable for large samples due to their computational complexity and memory requirements. 
We introduce a novel random feature regression algorithm that allows us (when necessary) to scale to virtually \textit{infinite} numbers of random features. We illustrate the performance of our method on the CIFAR-10 dataset.
\end{abstract}
\noindent

\bigskip\bigskip\bigskip

\noindent

\thispagestyle{empty}

\renewcommand{\thefootnote}{\number\value{footnote}}

\pagenumbering{arabic}
\def\baselinestretch{1.617}\small\normalsize%

\clearpage

\section{Introduction}
\label{sec:introduction}

Modern neural networks operate in the over-parametrized regime, which sometimes requires orders of magnitude more parameters than training data points. Effectively, they are {\it interpolators} (see, \cite{belkin2021fit}) and overfit the data in the training sample, with no consequences for the out-of-sample performance. This seemingly counterintuitive phenomenon is sometimes called ``benign overfit'' \citep{bartlett2020benign, tsigler2020benign}.  

In the so-called lazy training regime \cite{chizat2019lazy}, wide neural networks (many nodes in each layer) are effectively kernel regressions, and ``early stopping'' commonly used in neural network training is closely related to ridge regularization \citep{ali2019continuous}. See, 
\cite{jacot2018neural,hastie2019surprises,du2018gradient,du2019gradient,allen2019convergence}.  Recent research also emphasizes the ``double descent," in which expected forecast error drops in the high-complexity regime. See, for example,  \cite{zhang2016understanding,belkin2018reconciling,belkin2019does, spigler2019jamming, belkin2020two}.

These discoveries made many researchers argue that we need to gain a deeper understanding of kernel methods (and, hence, random feature regressions) and their link to deep learning. See, e.g., \cite{belkin2018understand}. Several recent papers have developed numerical algorithms for scaling kernel-type methods to large datasets and large numbers of random features. See, e.g., \cite{NEURIPS2021_08ae6a26, ma2017diving,arora2019exact, shankar2020neural}. In particular, \cite{arora2019harnessing} show how NTK combined with the support vector machines (SVM) (see also \cite{fernandez2014we}) perform well on small data tasks relative to many competitors, including the highly over-parametrized ResNet-34. In particular, while modern deep neural networks do generalize on small datasets (see, e.g., \cite{olson2018modern}), \cite{arora2019harnessing} show that kernel-based methods achieve superior performance in such small data environments. Similarly, \cite{du2019graph} find that the graph neural tangent kernel (GNTK) dominates graph neural networks on datasets with up to 5000 samples. \cite{shankar2020neural} show that, while NTK is a powerful kernel, it is possible to build other classes of kernels (they call Neural Kernels) that are even more powerful and are often at par with extremely complex deep neural networks.

In this paper, we develop a novel form of kernel ridge regression that can be applied to any kernel and any way of generating random features. We use a doubly stochastic method similar to that in \cite{dai2014scalable}, with an important caveat: We generate (potentially large,  defined by the RAM constraints) batches of random features and then use linear algebraic properties of covariance matrices to recursively update the eigenvalue decomposition of the feature covariance matrix, allowing us to perform the optimization in one shot across a large grid of ridge parameters. 

The paper is organized as follows. Section \ref{rel-work} discusses related work. In Section \ref{sec:rf_regression}, we provide a novel random feature regression mathematical formulation and algorithm. Then, Section \ref{sec:numerical_results} and Section \ref{sec:conclusion} present numerical results and conclusions, respectively.

\section{Related Work}\label{rel-work}

Before the formal introduction of the NTK in \cite{jacot2018neural}, numerous papers discussed the intriguing connections between infinitely wide neural networks and kernel methods. See, e.g., \cite{neal1996priors}; \cite{williams1997computing}; \cite{le2007continuous}; \cite{hazan2015steps}; \cite{lee2018deep}; \cite{matthews2018gaussian}; \cite{novak2018bayesian}; \cite{garriga2018deep}; \cite{cho2009kernel}; \cite{daniely2016toward}; \cite{daniely2017sgd}. As in the standard random feature approximation of the kernel ridge regression (see \cite{rahimi2007random}), only the network's last layer is trained in the standard kernel ridge regression. A surprising discovery of \cite{jacot2018neural}  
is that (infinitely) wide neural networks in the lazy training regime converge to a kernel even though all network layers are trained. The corresponding kernel, the NTK, has a complex structure dependent on the neural network's architecture. See also \cite{lee2019wide}, \cite{arora2019exact} for more results about the link between NTK and the underlying neural network, and \cite{novak2019neural} for an efficient algorithm for implementing the NTK. In a recent paper, \cite{shankar2020neural} introduce a new class of kernels and show that they perform remarkably well on even very large datasets, achieving a 90\% accuracy on the CIFAR-10 dataset. While this performance is striking, it comes at a huge computational cost. \cite{shankar2020neural} write: 

{\it ``CIFAR-10/CIFAR-100 consist of $60,000$ $32\times 32\times 3$ images and MNIST consists of $70,000$ $28\times 28$ images. Even with this constraint, the largest compositional 
kernel matrices we study took approximately 1000 GPU hours to compute. Thus, we believe an imperative direction of future work is reducing the complexity of each kernel evaluation. Random feature methods or other compression schemes could play a significant role here.} 

In this paper, we offer one such highly scalable scheme based on random features. However, computing the random features underlying the Neural Kernels of \cite{shankar2020neural} would require developing non-trivial numerical algorithms based on the recursive iteration of non-linear functions. We leave this as an important direction for future research. 

As in standard kernel ridge regressions, we train our random feature regression on the full sample. This is a key computational limitation for large datasets. After all, one of the reasons for the success of modern deep learning is the possibility of training them using stochastic gradient descent on mini-batches of data. \citet{ma2017diving} shows how mini-batch training can be applied to kernel ridge regression. A key technical difficulty arises because kernel matrices (equivalently, covariance matrices of random features) have eigenvalues that decay very quickly. Yet, these low eigenvalues contain essential information and cannot be neglected. Our regression method can be easily modified to allow for mini-batches.
Furthermore, it is known that mini-batch linear regression can even lead to performance gains in the high-complexity regime. As \cite{lejeune2020implicit} show, one can run regression on mini-batches and then treat the obtained predictions as an ensemble. \cite{lejeune2020implicit} prove that, under technical conditions, the average of these predictions attains a lower generalization error than the full-train-sample-based regression. 
We test this mini-batch ensemble approach using our method and show that, indeed, with moderately-sized mini-batches, the method's performance matches that of the full sample regression.

Moreover, there is an intriguing connection between mini-batch regressions and spectral dimensionality reduction. By construction, the feature covariance matrix with a mini-batch of size $B$ has at most $B$ non-zero eigenvalues. Thus, a mini-batch effectively performs a dimensionality reduction on the covariance matrix. Intuitively, we expect that the two methods (using a mini-batch of size $B$ or using the full sample but only keeping $B$ largest eigenvalues) should achieve comparable performance. We show that this is indeed the case for small sample sizes. However, the spectral method for larger-sized samples ($N \geq 10000$) is superior to the mini-batch method unless we use very large mini-batches. For example, on the full CIFAR-10 dataset, the spectral method outperforms the mini-batch approach by 3\% (see Section~\ref{sec:numerical_results} for details).


\section{Random Features Ridge Regression and Classification} 
\label{sec:rf_regression}
Suppose that we have a train sample $(X,y)\ =\ (x_i,y_i)_{i=1}^N,\ x_i\in \R^d,\ y_i\in \R,$ so that $X\in \R^{N\times d},\ y\in \R^{N\times 1}.$  Following \cite{rahimi2007random} we construct a large number of random features $f(x;\theta_p),\ p=1,\ldots,P,$ where $f$ is a non-linear function and $\theta_p$ are sampled from some distribution, and $P$ is a large number. We denote $S=f(X;\theta)\in \R^{N\times P}$ as the train sample realizations of random features. Following \cite{rahimi2007random}, we consider the random features ridge regression,
\begin{equation}
\beta(z)\ =\ (S^{\top}S/N+zI)^{-1}S^{\top}y/N\,, 
\end{equation}
as an approximation for kernel ridge regression when $P\to\infty$. For classification problems, it is common to use categorical cross-entropy as the objective. However, as \citet{belkin2021fit} explains, minimizing the mean-squared error with one-hot encoding often achieves superior generalization performance. Here, we follow this approach. Given the $K$ labels, $k=1,\ldots,K,$ we build the one-hot encoding matrix $Q=(q_{i,k})$ where $q_{i,k}={\bf 1}_{y_i=k}.$ Then, we get 
\begin{equation}
\beta(z)\ =\ (S^{\top}S/N+zI)^{-1}S^{\top}Q/N\ \in\ \R^{P\times K}\,.
\end{equation}
Then, for each test feature vector ${\bf s}=f(\bx;\theta) \in \R^P,$ we get a vector $\beta(z)^\top {\bf s} \in \R^{K}$. Next, define the actual classifier as
\begin{equation}
k(\bx;z)\ =\ \arg\max\{\beta(z)^\top {\bf s} \}\ \in\ \{1,\cdots,K\}\,. 
\end{equation}

\subsection{Dealing with High-Dimensional Features} 
\label{sec:high_dimensional}

A key computational (hardware) limitation of kernel methods comes from the fact that, when $P$ is large, computing the matrix $S^{\top}S\in \R^{P\times P}$ becomes prohibitively expensive, in particular, because $S$ cannot even be stored in RAM. We start with a simple observation that the following identity implies that storing all these features is not necessary:\footnote{This identity follows directly from $(S^{\top}S/N+zI)S^{\top}=S^{\top}(SS^{\top}/N+zI).$} 
\begin{equation}
(S^{\top}S/N+zI)^{-1}S^{\top}\ =\ S^{\top}(SS^{\top}/N+zI)^{-1}\,,
\end{equation}
and therefore we can compute $\beta(z)$ as 
\begin{equation}
\beta(z)\ =\ S^{\top}(SS^{\top}/N+zI)^{-1}y/N\,. 
\end{equation}

Suppose now we split $S$ into multiple blocks, $S_1,\ldots,S_K,$ where $S_k\in \R^{N\times P_1}$ for all $k=1, \ldots, K$, for some small $P_1,$ with $KP_1=P.$ Then, 
\begin{equation}
\Psi = SS^{\top}\ =\ \sum_{k=1}^K S_kS_k^{\top}
\end{equation}
can be computed by generating the blocks $S_k,$ one at a time, and recursively adding $S_kS_k^{\top}$ up. Once $\Psi$ has been computed, one can calculate its eigenvalue decomposition, $\Psi=VDV^{\top},$ and then evaluate $Q(z)=(\Psi/N+zI)^{-1}y/N\ =\ V(D+zI)^{-1}V^{\top}y/N\in \R^N$ in one go for a grid of $z.$ Then, using the same seeds, we can again generate the random features $S_k$ and compute $\beta_k(z)=S_k^{\top}Q(z)\in \R^{P_1}.$ Then, $\beta(z)=(\beta_k(z))_{k=1}^K\in \R^P\,.$ The logic described above is formalized in Algorithm \ref{alg:giant_regression}.

\begin{algorithm}
\caption{FABReg}
\label{alg:giant_regression}
\begin{algorithmic}
\Require $P_1$, $P$, $X \in \R^{N\times d}$, $y \in \R^{N}$, $z$, $voc\_curve$ 
\State $blocks \gets P // P_1$
\State $k \gets 0$
\State $\Psi \gets 0_{N \times N}$
\While{$k < blocks$}
    \State Generate $S_{k} \in \R^{N\times P_1}$ {Use $k$ as seed}
    
    \State $\Psi \gets \Psi + S_kS_k\top$
    \If{$k$ in $voc\_curve$}
        \State $DV \gets eigen(\frac{\Psi}{N})$
        \State $Q_k(z) \gets V (D + zI)^{-1} V^\top \frac{y}{N}$ \Comment{Store $Q_k(z)$}
 
    \EndIf
    \State $k = k + 1$
\EndWhile
\State $DV \gets eigen(\frac{\Psi}{N})$
\State $Q(z) \gets V (D + zI)^{-1} V^\top \frac{y}{N}$
\State $k \gets 0$
\While{$k < blocks$}
    \State (re-)Generate $S_{k} \in \R^{N\times P_1}$ \Comment{Use $k$ as seed}
    \State $\beta_k(z) \gets S^{\top}_kQ(z)$ 
    \State $\hat{y} \mathrel{+}=  S_k\beta_k$

\EndWhile
\end{algorithmic}
\end{algorithm}

\subsection{Dealing with Massive Datasets}
 
The above algorithm relies crucially on the assumption that $N$ is small. 
Suppose now that the sample size $N$ is so large that storing and eigen-decomposing the matrix $SS^{\top}\in \R^{N\times N}$ becomes prohibitively expensive. In this case, we proceed as follows. 

Define for all $k = 1, \ldots, K$
\begin{equation}
\Psi_k\ =\  \sum_{\kappa=1}^k S_kS_k^{\top}\in \R^{N\times N},\ \Psi_0=0_{N\times N}\,,
\end{equation}
and let $\gl_1(A)\ge\cdots\ge \gl_N(A)$ be the eigenvalues of a symmetric matrix $A\in \R^{N\times N}.$ Our goal is to design an approximation to $(\Psi_K+zI)^{-1},$ based on a simple observation that the eigenvalues of the empirically observed $\Psi_k$ matrices tend to decay very quickly, with only a few hundreds of largest eigenvalues being significantly different from zero. In this case, we can fix a $\nu\in \mathbb N$ and design a simple, rank$-\nu$  approximation to $\Psi_K$ by annihilating all eigenvalues below $\gl_\nu(\Psi_K).$  As we now show, it is possible to design a recursive algorithm for constructing such an approximation to $\Psi_K,$ dealing with small subsets of random features simultaneously. To this end, we proceed as follows. 

Suppose we have constructed an approximation $\hat\Psi_k\in \R^{N\times N}$ to $\Psi_k$ with rank $\nu,$ and let $V_k\in \R^{N\times\nu}$ be the corresponding matrix of orthogonal eigenvectors for the non-zero eigenvalues, and $D_k\in \R^{\nu\times \nu}$ the diagonal  matrix of eigenvalues so that $\hat\Psi_k=V_kD_kV_k^{\top}$ and $V_k^{\top}V_k=I_{\nu\times\nu}$. Instead of storing the full $\hat\Psi_k$ matrix, we only need to store the pair $(V_k, D_k).$ For all $k=1,\ldots,K$, we now define 
\begin{equation}
\tilde\Psi_{k+1}\ =\ \hat\Psi_k\ +\ S_{k+1}S_{k+1}^{\top}\,.
\end{equation}
This $N\times N$ matrix is a theoretical construct. We never actually compute it (see Algorithm \ref{alg:spectral}).
Let $\Theta_k=I-V_kV_k^{\top}$ be the orthogonal projection on the kernel of $\hat\Psi_k,$ and 
\begin{equation}
\tilde S_{k+1}\ =\ \Theta_k S_{k+1}\ =\ S_{k+1}-\underbrace{V_k}_{N\times\nu}(\underbrace{V_k^{\top}S_{k+1}}_{\nu\times P_1})
\end{equation}
be $S_{k+1}$ orthogonalized with respect to the columns of $V_k.$ Then, we define $\tilde W_{k+1}=\tilde S_{k+1} (\tilde S_{k+1}\tilde S_{k+1}^{\top})^{-1/2}$ to be the orthogonalized columns of $\tilde S_{k+1},$ and $\hat V_{k+1}=[V_k, \tilde W_{k+1}]$. To compute $\tilde S_{k+1} (\tilde S_{k+1}\tilde S_{k+1}^{\top})^{-1/2},$ we use the following lemma that, once again, uses smart eigenvalue decomposition techniques to avoid dealing with the $N\times N$ matrix $\tilde S_{k+1}\tilde S_{k+1}^{\top}$.

\begin{lemma}\label{spec-trans} Let $\underbrace{\tilde S_{k+1}^\top\tilde S_{k+1}}_{\nu\times \nu}=W \gd \tilde W^\top$ be the eigenvalue decomposition of $\tilde S_{k+1}^\top\tilde S_{k+1}.$ Then, $\tilde W = \tilde S_{k+1}W \gd^{-1/2}$ is the matrix of eigenvectors of $\tilde S_{k+1}\tilde S_{k+1}^\top$ for the non-zero eigenvalues. Thus, 
\begin{equation}
\tilde S_{k+1} (\tilde S_{k+1}\tilde S_{k+1}^{\top})^{-1/2}\ =\ \tilde W_{k+1}\,. 
\end{equation}
\end{lemma}

By construction, the columns of $\hat V_{k+1}$ form an orthogonal basis of the span of the columns of $V_k,S_{k+1},$ and hence 
\begin{equation}
\Psi_{k+1,*}\ =\ \hat V_{k+1}^{\top}\tilde\Psi_{k+1}\hat V_{k+1}\ \in\ \R^{(P_1+\nu)\times (P_1+\nu)}
\end{equation}
has the same non-zero eigenvalues as $\tilde\Psi_{k+1}.$ We then define $\tilde V_{k+1}\in\R^{(P_1+\nu)\times\nu}$ to be the matrix with eigenvectors of $\Psi_{k+1,*}$ for the largest $\nu$ eigenvalues, and we denote the diagonal matrix of these eigenvalues by $D_{k+1}\in \R^{\nu\times\nu}$, and then we define $V_{k+1}\ =\ \hat V_{k+1}\tilde V_{k+1}\,.$
Then, $\hat\Psi_{k+1}\ =\ V_{k+1}D_{k+1}V_{k+1}\ =\ \Pi_{k+1}\tilde\Psi_{k+1}\Pi_{k+1}\,,$
where $\Pi_{k+1}\ =\ \hat V_{k+1}\tilde V_{k+1}\tilde V_{k+1}^{\top}\hat V_{k+1}^{\top}$
is the orthogonal projection onto the eigen-subspace of $\tilde\Psi_{k+1}$ for the largest $\nu$ eigenvalues. 

\begin{lemma}\label{bound} We have $\hat\Psi_k\le\tilde\Psi_k\le \Psi_K$ and 
\begin{equation}
\|\Psi_{k}-\hat\Psi_{k}\|\ \le\ \sum_{i=1}^k\gl_{\nu+1}(\Psi_i)\ \le\ k\,\gl_{\nu+1}(\Psi_K)\,,
\end{equation}
and 
\begin{equation}
\|(\Psi_{k+1}+zI)^{-1}-(\hat\Psi_{k+1}+zI)^{-1}\|\ \le\ z^{-2}\sum_{i=1}^k\gl_{\nu+1}(\Psi_i)\,. 
\end{equation}
\end{lemma}

There is another important aspect of our algorithm: It allows us to directly compute the performance of models with an expanding level of complexity. Indeed, since we load random features in batches of size $P_1,$ we generate predictions for $P\in [P_1,2P_1,\cdots, KP_1].$ This is useful because we might use it to calibrate the optimal degree of complexity and because we can directly study the double descent-like phenomena, see, e.g., \cite{belkin2018reconciling} and \cite{nakkiran2021deep}. That is the effect of complexity on the generalization error. In the next section, we do this. As we show, consistent with recent theoretical results \cite{kelly2022virtue}, with sufficient shrinkage, the double descent curve disappears, and the performance becomes almost monotonic in  complexity. Following \cite{kelly2022virtue}, we name this phenomenon {\it the virtue of complexity (VoC)} and the corresponding performance plots {\it the VoC curves.}  See, Figure \ref{fig:small_cifar_voc_curve} below. 

We call this algorithm Fast Annihilating Batch Regression (FABReg) as it annihilates all eigenvalues below $\lambda_{\nu}(\Psi_K)$ and allows to solve the random features ridge regression in one go for a grid of $z$. Algorithm~\ref{alg:spectral} formalizes the logic described above.

\begin{algorithm}
\caption{FABReg-$\nu$}
\label{alg:spectral}
\begin{algorithmic}
\Require $\nu$, $P_1$, $P$, $X \in \R^{N\times d}$, $y \in \R^{N}$, $z$, $voc\_curve$ 
\State $blocks \gets P // P_1$
\State $k \gets 0$
\While{$k < blocks$}
\State Generate $S_{k} \in \R^{N\times P_1}$ \Comment{Use $k$ as seed to generate the random features}
\If{$k = 0$} 
    \State $\tilde{d}, \tilde{V} \gets eigen(S^{\top}_{k} S_{k})$ 
    \State  $V \gets S_{k}\tilde{V}diag(\tilde{d})^{-\frac{1}{2}}$
    \State $V_0 \gets V_{:,min(\nu, P_1)}$ \Comment{Save $V_0$}
    \State $d_0 \gets \tilde{d}_{:min(\nu, P_1)}$ \Comment{Save $d_0$}
    \If{k in $voc\_curve$}
        \State $Q_0(z) \gets V_0 (diag(d_0) + zI)^{-1}V_0^{\top} y$ \Comment{Save $Q_0(z)$}
    \EndIf
\ElsIf{$k > 0$}
    \State $\tilde{S_k} \gets (I - V_{k-1}V^{\top}_{k-1})S_k$
    \State $\Gamma_k \gets \tilde{S^{\top}_k}\tilde{S_k}$
    \State $\delta_k, W_k \gets eigen(\Gamma_k)$
    \State Keep top min($\nu, P1$) eigenvalues and eigenvectors from $\delta_k, W_k$
    \State $\tilde{W_k} \gets \tilde{S_k} W_k diag(\delta_k)^{-\frac{1}{2}}$
    \State $\hat{V_k} \gets [V_{k-1}, \tilde{W_k}]$
    \State $\bar{V_k} \gets \hat{V^{\top}_k} V_{k-1}$
    \State $\bar{W_k} \gets \bar{V_k} diag(d_{k-1}) \bar{V^{\top}_k}$
    \State $\bar{S_k} \gets \hat{V^{\top}_k}S_k$
    \State $\bar{Z_k} \gets \bar{S_k}S^{\top}_k$
    \State $\Psi_{*} \gets \bar{W_k} \bar{Z_k}$
    \State $d_k, V_k \gets eigen(\Psi_{*})$
    \State Keep top min($\nu, P1$) eigenvalues and eigenvectors from $d_k, V_k$
    \State $V_k \gets \hat{V_k} V_k$ \Comment{Save $d_k, V_k$}
    \If{k in $voc\_curve$}
        \State $Q_k(z) \gets V_k (diag(d_k) + zI)^{-1}V_k^{\top} y$ \Comment{Save $Q_k(z)$}
    \EndIf
\EndIf
\State $k = k + 1$
\EndWhile

\State $k \gets 0$
\While{$k < blocks$}
    \State (re-)Generate $S_{k} \in \R^{N\times P_1}$ \Comment{Use $k$ as seed to generate the random features}
    \State $\beta_k(z) \gets S^{\top}_kQ_k(z)$ 
    \State $\hat{y} \mathrel{+}=  S_k\beta_k$

\EndWhile
\end{algorithmic}
\end{algorithm}

\section{Numerical Results}
\label{sec:numerical_results}
This section presents several experimental results on different datasets to evaluate FABReg's performance and applications. In contrast to the most recent computational power demand in kernel methods, e.g., \cite{shankar2020neural}, we ran all experiments on a laptop, a MacBook Pro model A2485, equipped with an M1 Max with a 10-core CPU and 32 GB RAM. 


\subsection{A comparison with sklearn}
\label{sec:sklearn}
We now aim to show FABReg's training and prediction time with respect to the number of features~$d$. To this end, we do not use any random feature projection or the rank-$\nu$ matrix approximation described in Section~\ref{sec:high_dimensional}. We draw $N = 5000$ i.i.d. samples from $\otimes_{j=1}^d\cN(0,1) $ and let
\begin{equation*}
    y_i =  x_i \beta  + \epsilon_i\quad \forall i = 1,\ldots, N,
\end{equation*}
where $\beta \sim \otimes_{j=1}^d \cN(0,1)$, and $\epsilon_i \sim \cN(0,1)$ for all~$i = 1,\ldots, N$. 
Then, we define
\begin{equation*}
y_i = 
\begin{cases} 
      1 & \text{if $y_i > \text{median}(\boldsymbol{y})$}, \\
      0 & \text{otherwise}
\end{cases}\quad \forall i= 1,\ldots, N.
\end{equation*}
Next, we create a set of datasets for classification with varying complexity $d$ and keep the first 4000 samples as the training set and the remaining 1000 as the test set. We show in Figure~\ref{fig:sklearn_comparison} the average training and prediction time (in seconds) of FABReg with a different number of regularizers~( we denote this number by $|z|$) and {\it sklearn} RidgeClassifier with an increasing number of features~$d$. The training and prediction time is averaged over five independent runs. As one can see, our method is drastically faster when $d > 10000$. E.g., for $d=100000$ we outperform {\it sklearn} by approximately 5 and 25 times for $|z|=5$ and  $|z|=50$, respectively. Moreover, one can notice that the number of different shrinkages $|z|$ does not affect FABReg. We report a more detailed table with average training and prediction time and standard deviation in Appendix \ref{sec:additional_results}.
\begin{figure}
    \centering
    \includegraphics[width=0.45\textwidth]{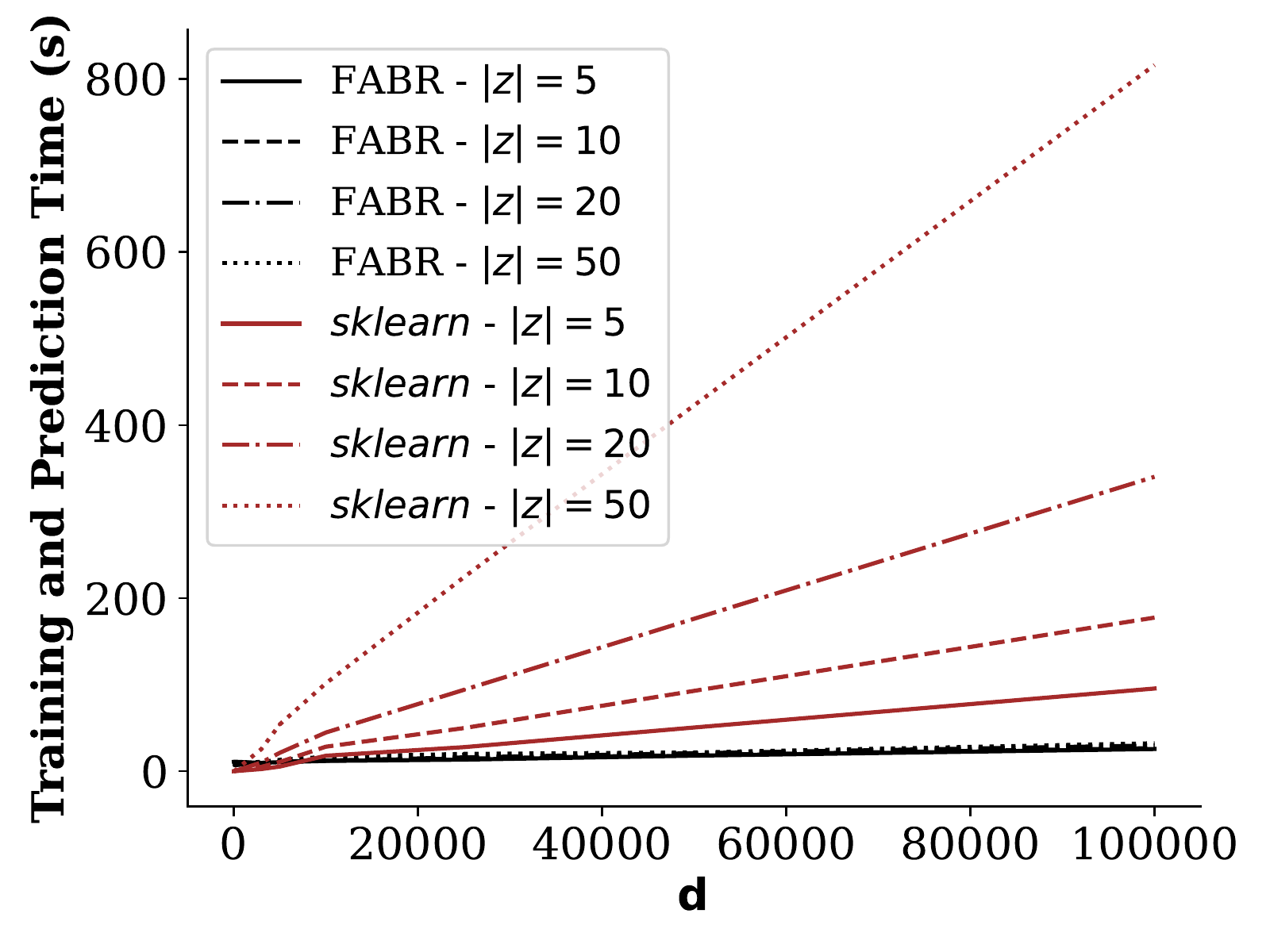}
    \caption{The figure above compares FABReg training and prediction time, shown on the y-axis, in black, against {\it sklearn}'s RidgeClassifier, in red, for an increasing amount of features, shown on the x-axis, and the number of shrinkages $z$. Here, $|z|$ denotes the number of different values of $z$ for which we perform the training.}
    \label{fig:sklearn_comparison}
\end{figure}

\subsection{Experiments on Real Datasets}    
We assess FABReg's performance on both small and big datasets regimes for further evaluation. For all experiments, we perform a random features kernel ridge regression for demeaned one-hot labels and solve the optimization problem using FABReg as described in Section~\ref{sec:rf_regression}. 

\subsubsection{Data Representation}
\input{small_dataset_cifar}
\begin{figure}[t]
    \centering
    \includegraphics[scale=0.45]{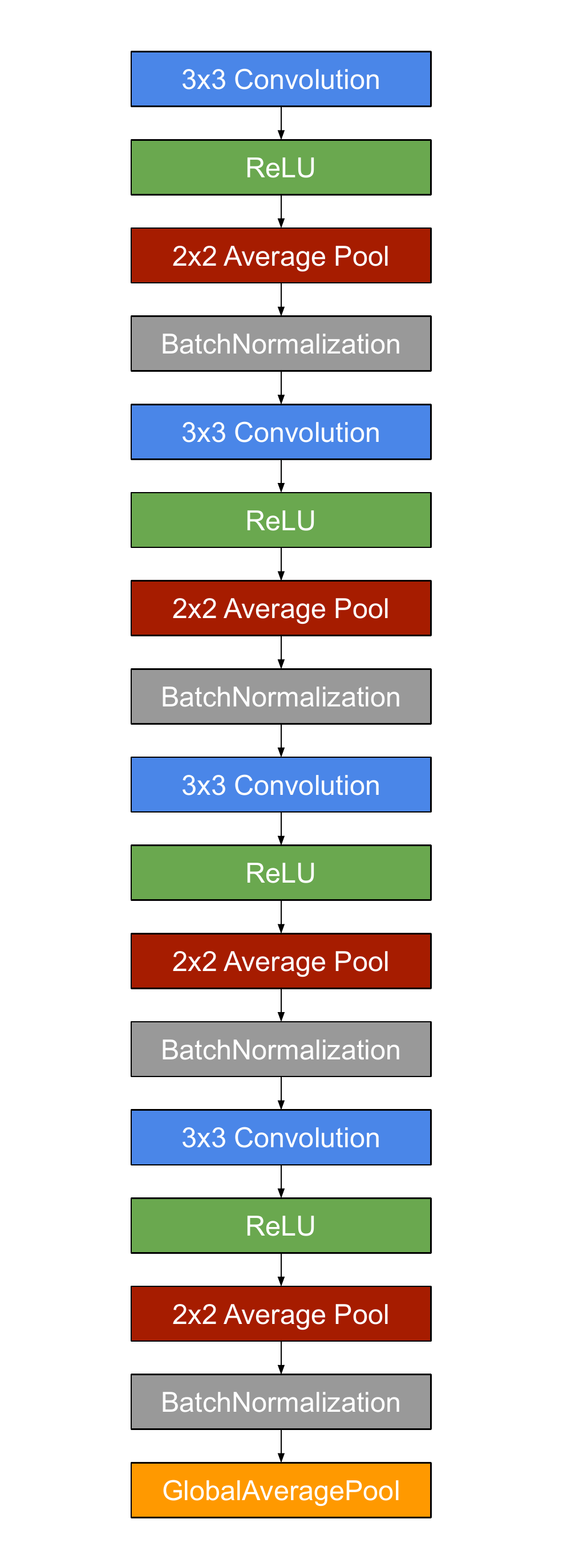}
    \caption{CNN architecture used to extract image features.}
    \label{fig:cnn_architecture}
\end{figure}
FABReg requires, like any standard kernel methods or randomized-feature techniques, a good data representation. Usually, we don't know such a representation {\it a-priori}, and learning a good kernel is outside the scope of this paper. Therefore, we build a simple Convolutional Neural Network (CNN) mapping $h: \R^d \to \R^D$; that extracts image features $\tilde{x} \in \R^{D}$ for some sample $x \in \R^{d}$. The CNN is not optimized; we use it as a simple random feature mapping. The CNN architecture, shown in Fig. \ref{fig:cnn_architecture}, alternates a $3 \times 3$ convolution layer with a {\it ReLU} activation function, a $2 \times 2$ Average Pool, and a BatchNormalization layer \cite{ioffe2015batch}. Convolutional layers weights are initialized using He Uniform \cite{he2015delving}. To vectorize images, we use a global average pooling layer that has proven to enforce correspondences between
feature maps and to be more robust to spatial translations of the input \cite{lin2013network}. We finally obtain the train and test random features realizations $s = f(\tilde{x}, \theta)$. Specifically, we use the following random features mapping
\begin{equation}
    s_i = \sigma(W\tilde{x}),
\end{equation}
where $W \in \R^{P \times D}$ with $w_{i,j} \sim \cN(0,1)$ and $\sigma$ is some elementwise activation function. This can be described as a one-layer neural network with random weights $W$. To show the importance of over-parametrized models, throughout the results, we report the {\it complexity}, $c$, of the model as $c = P/N$, that is, the ratio between the parameters (dimensions) and the number of observations. See \cite{belkin2018reconciling, hastie2019surprises, kelly2022virtue}. 

\subsubsection{Small Datasets}
\label{sec:small_dataset}

We now study the performance of FABReg on the subsampled CIFAR-10 dataset \cite{krizhevsky2009learning}. To this end, we reproduce the same experiment described in \cite{arora2019harnessing}. In particular, we obtain random subsampled training set $(y; X) = (y_i; x_i)_{i=1}^n$ where $n \in \{10, 20, 40, 80, 160, 320, 640, 1280\}$ and test on the {\it whole} test set of size 10000. We make sure that exactly $n/10$ sample from each image class is in the training sample. We train FABReg using random features projection of the subsampled training set
\begin{equation*}
    S = \sigma(W g(X)) \in \R^{n \times P},
\end{equation*}
where $g$ is an untrained CNN from Figure \ref{fig:cnn_architecture}, randomly initialized using He Uniform distribution. In this experiment, we push the model complexity $c$ to 100; in other words, FABReg's number of parameters equals a hundred times the number of observations in the subsample. As $n$ is small, we deliberately do not perform any low-rank covariance matrix approximation. Finally, we run our model twenty times and report the mean out-of-sample performance and the standard deviation. We report in Table~\ref{table:small_dataset_cifar} FABReg's performance for different shrinkages ($z$) together with ResNet-34 and the 14-layers CNTK. Without any complicated random feature projection, FABReg can outperform both ResNet-34 and CNTK. FABReg's test accuracy increases with the model's complexity $c$ on different ($n$) subsampled CIFAR-10 datasets. We show Figure~\ref{fig:single_full_curve} as an example for $n=10$. Additionally, we show, to better observe the double descent phenomena, truncated curves at $c=25$ for all CIFAR-10 subsamples in Figure~\ref{fig:small_cifar_voc_curve_zoom}. The full curves are shown in Appendix \ref{sec:additional_results}. To sum up this section findings:
\begin{itemize}
    \item FABReg, with enough complexity together and a simple random feature projection, is able to outperform deep neural networks (ResNet-34) and CNTKs.
    \item FABReg always reaches the maximum accuracy beyond the interpolation threshold. 
    \item Moreover, if the random feature ridge regression shrinkage $z$ is sufficiently high, the double descent phenomenon disappears, and the accuracy does not drop at the interpolation threshold point, i.e., when $c=1$ or $n=P$. Following \cite{kelly2022virtue}, we call this phenomenon {\it virtue of complexity} (VoC).
\end{itemize}

\begin{figure}[]
    \centering
    \includegraphics[width=0.45\linewidth]{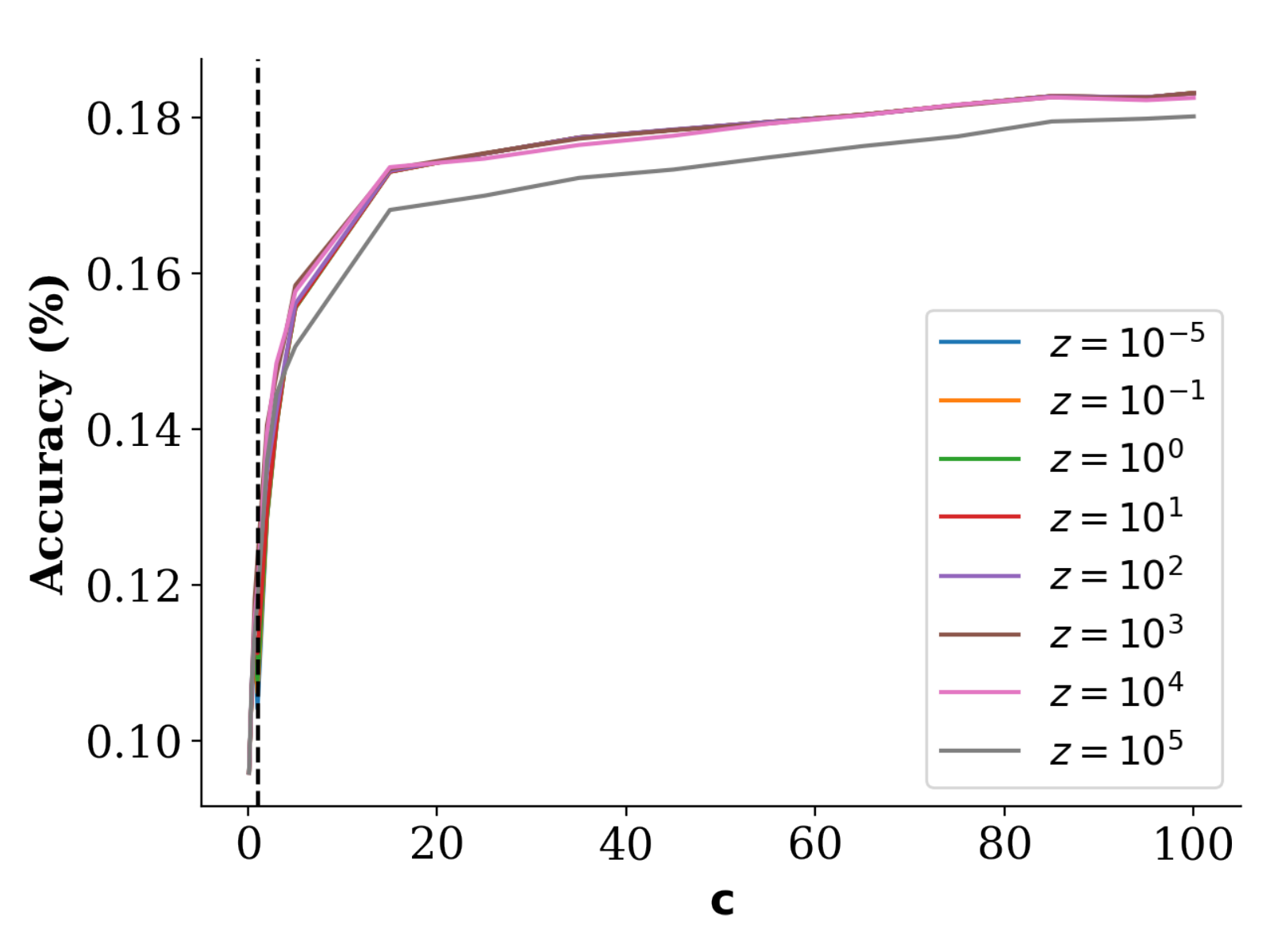}
    \caption{The figures above show FABReg's test accuracy increases with the model's complexity $c$ on the subsampled CIFAR-10 dataset for $n=10$. The test accuracy is averaged over five independent runs.}
    \label{fig:single_full_curve}
\end{figure}
 
\begin{figure*}[t]
\captionsetup[subfigure]{justification=centering}
    \centering
    \begin{subfigure}[b]{0.24\textwidth}
        \centering
        \includegraphics[width=\textwidth]{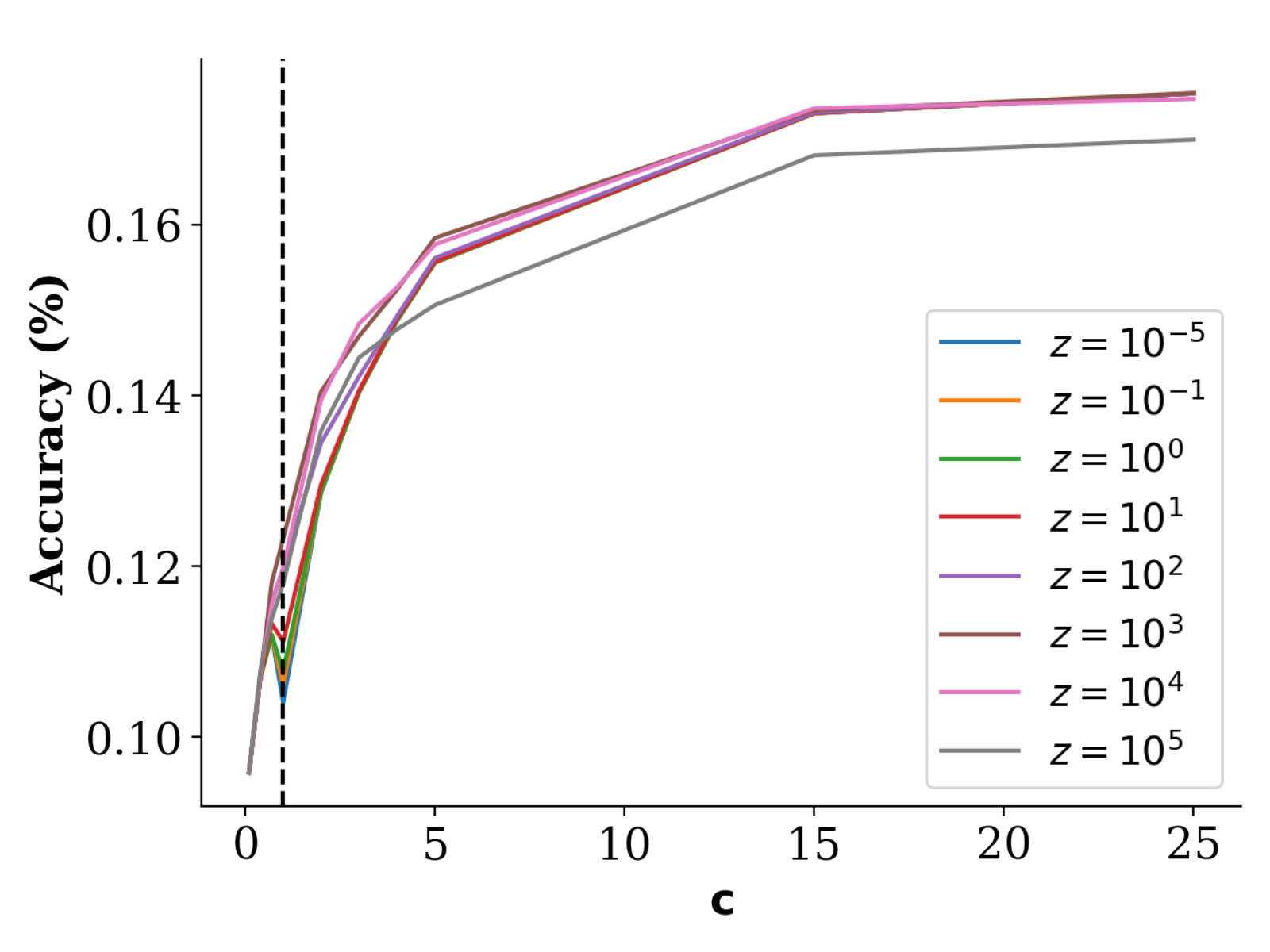}
        \caption{$n=10$}
        \label{fig:small_cifar_10_zoom}
    \end{subfigure}
    \hfill
    \begin{subfigure}[b]{0.24\textwidth}
        \centering
        \includegraphics[width=\textwidth]{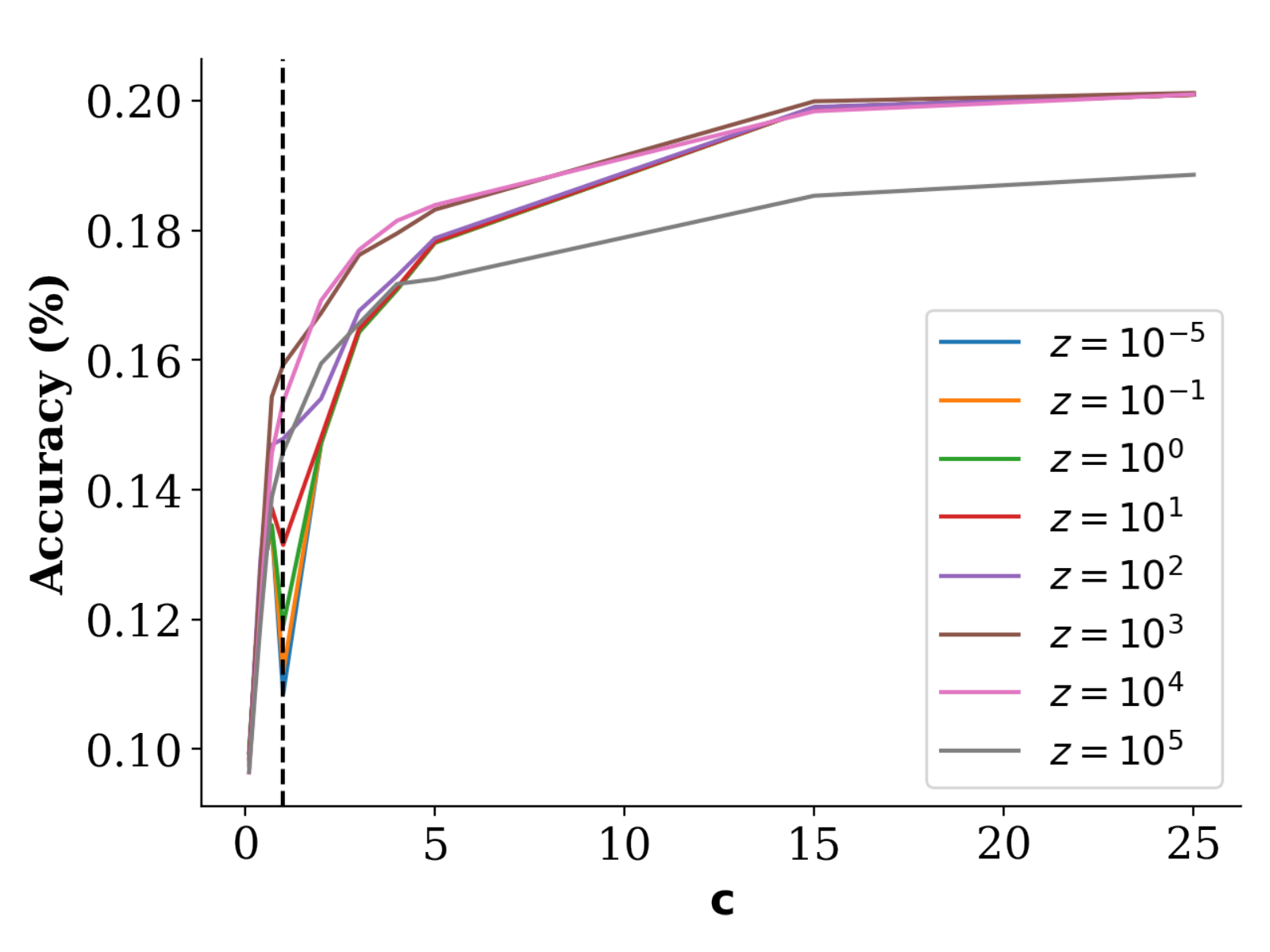}
        \caption{$n=20$}
        \label{fig:small_cifar_20_zoom}
    \end{subfigure}
    \hfill
    \begin{subfigure}[b]{0.24\textwidth}
        \centering
        \includegraphics[width=\textwidth]{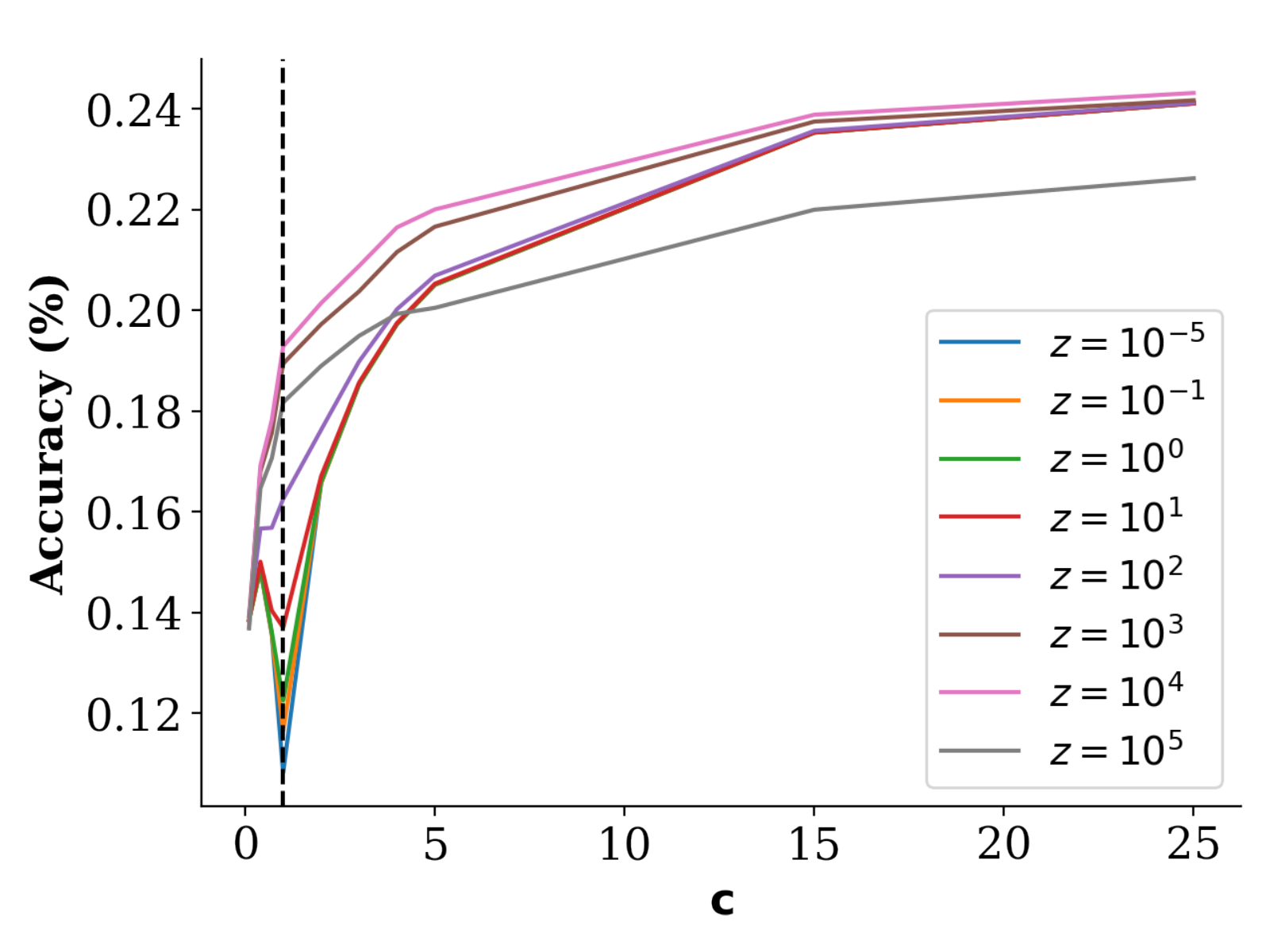}
        \caption{$n=40$}
        \label{fig:small_cifar_40_zoom}
    \end{subfigure}
    \hfill
    \begin{subfigure}[b]{0.24\textwidth}
        \centering
        \includegraphics[width=\textwidth]{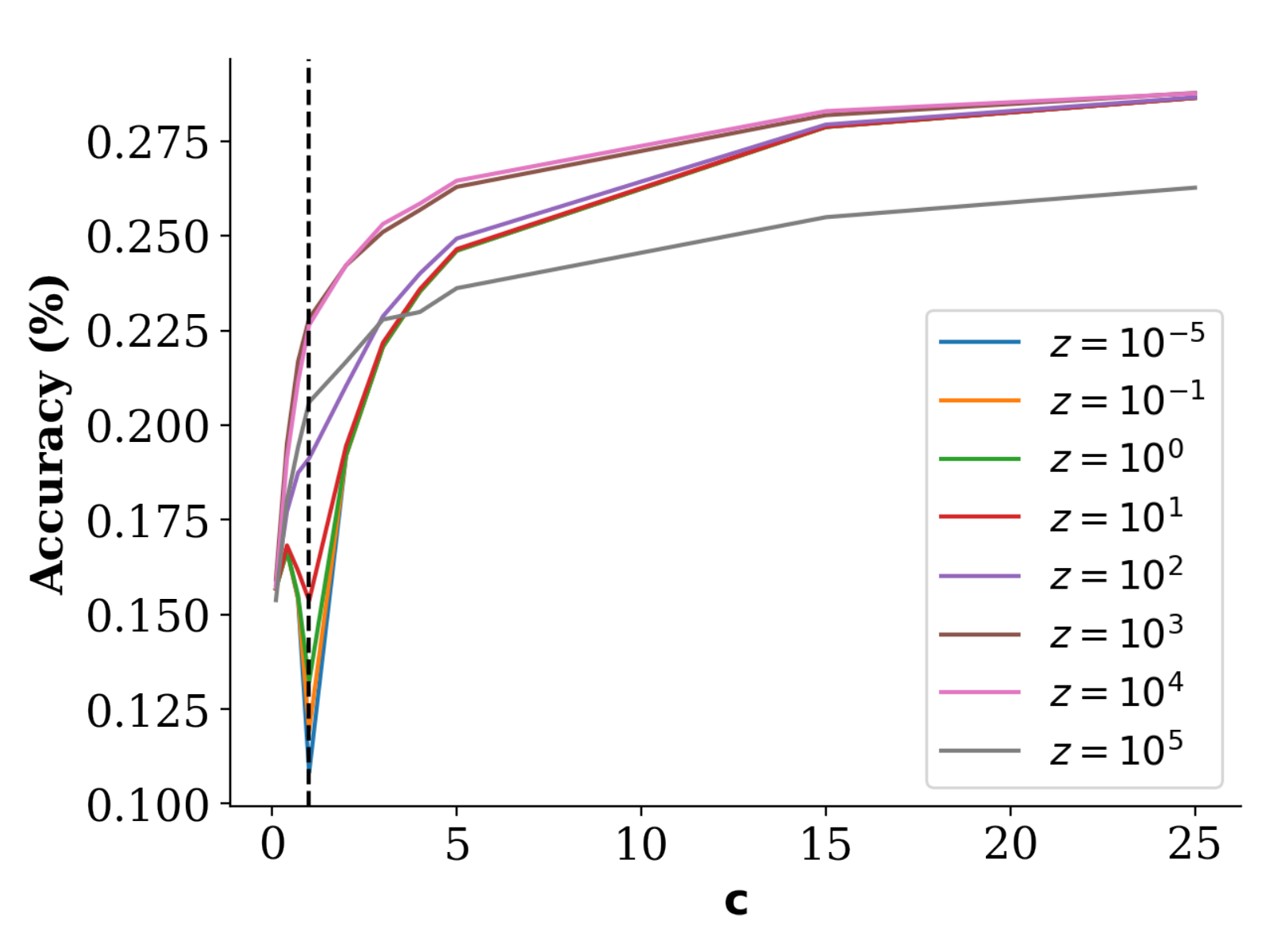}
        \caption{$n=80$}
        \label{fig:small_cifar_80_zoom}
    \end{subfigure}
    \begin{subfigure}[b]{0.24\textwidth}
        \centering
        \includegraphics[width=\textwidth]{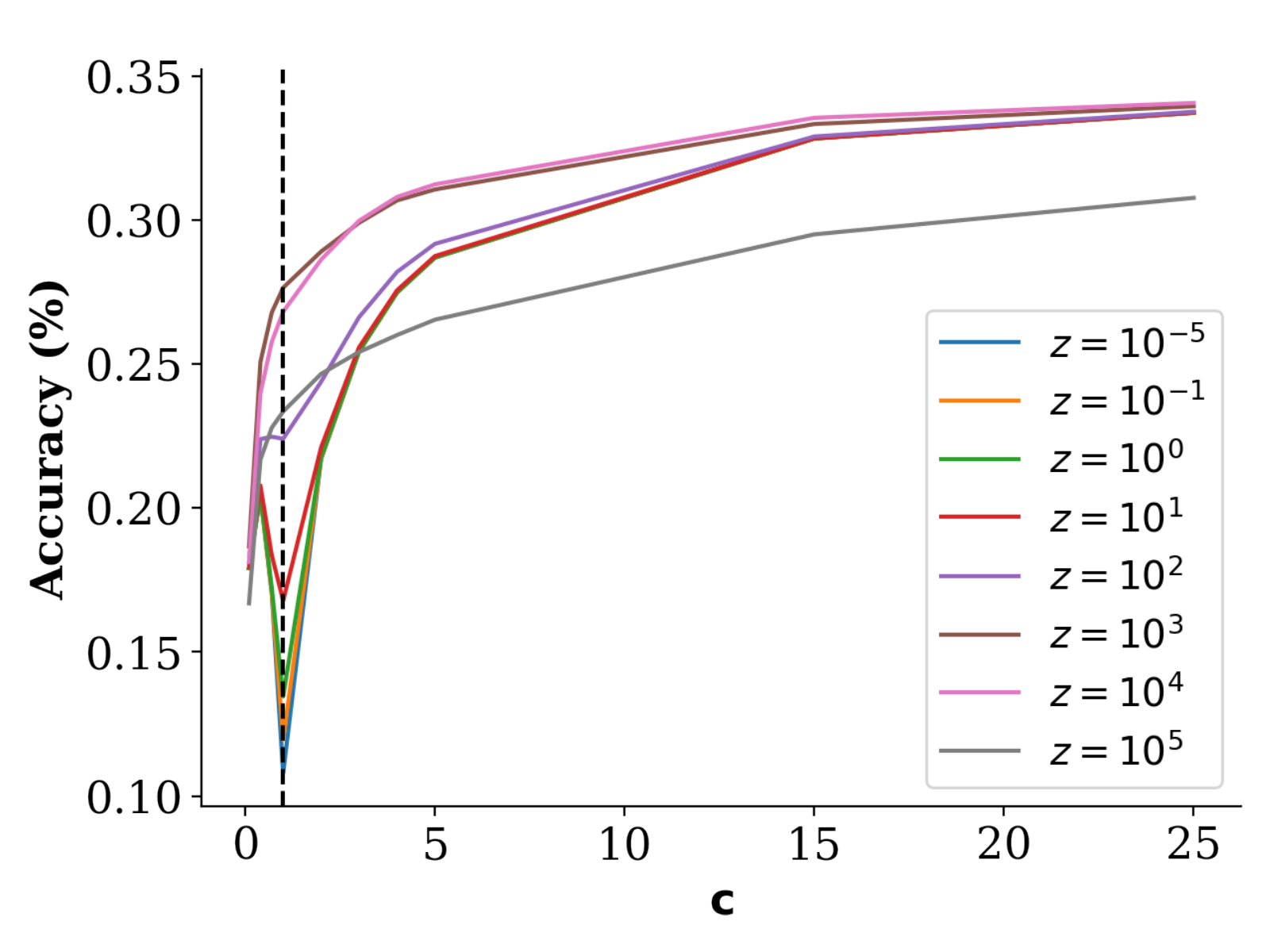}
        \caption{$n=160$}
        \label{fig:small_cifar_160_zoom}
    \end{subfigure}
    \hfill
    \begin{subfigure}[b]{0.24\textwidth}
        \centering
        \includegraphics[width=\textwidth]{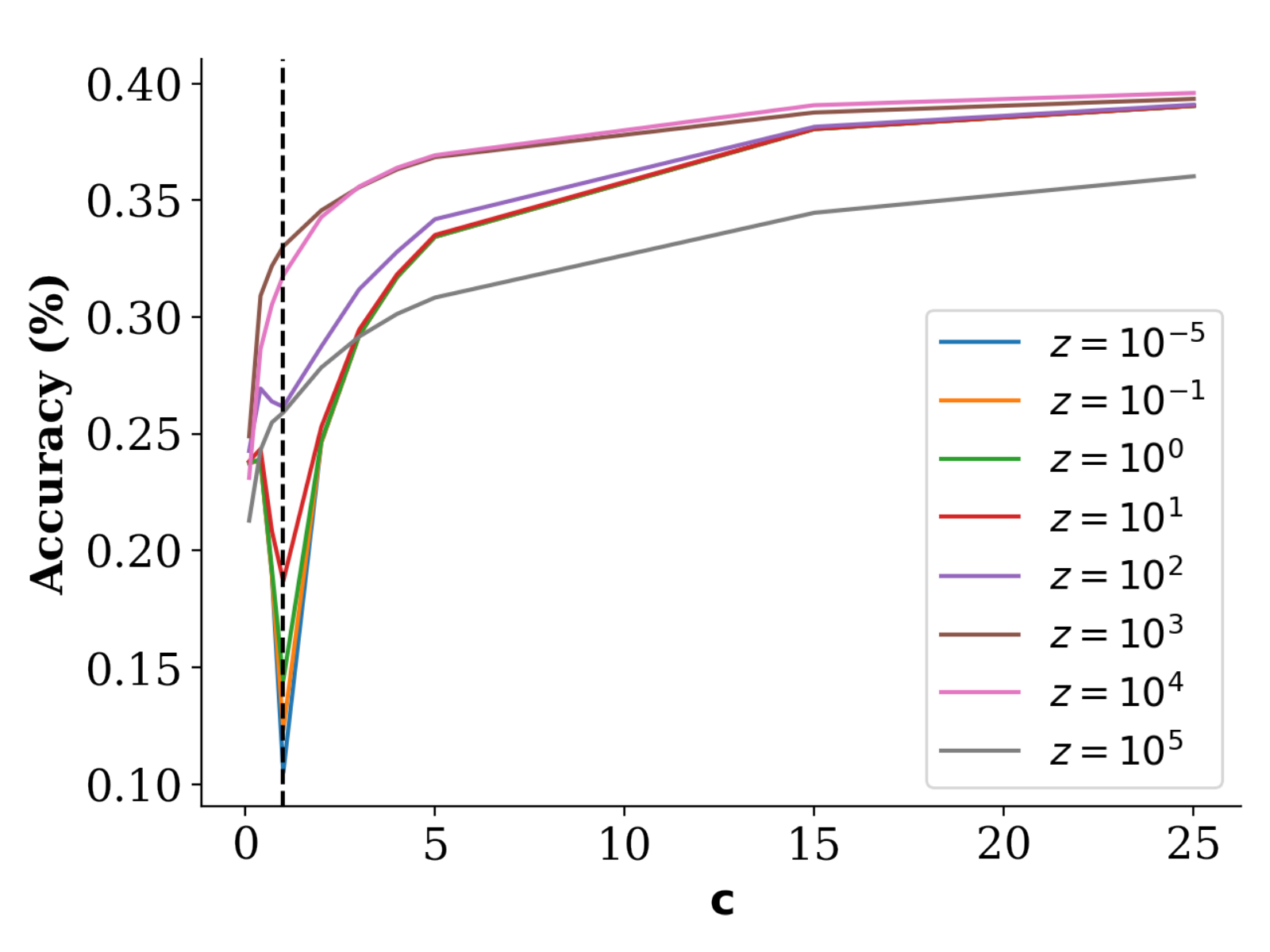}
        \caption{$n=320$}
        \label{fig:small_cifar_320_zoom}
    \end{subfigure}
    \hfill
    \begin{subfigure}[b]{0.24\textwidth}
        \centering
        \includegraphics[width=\textwidth]{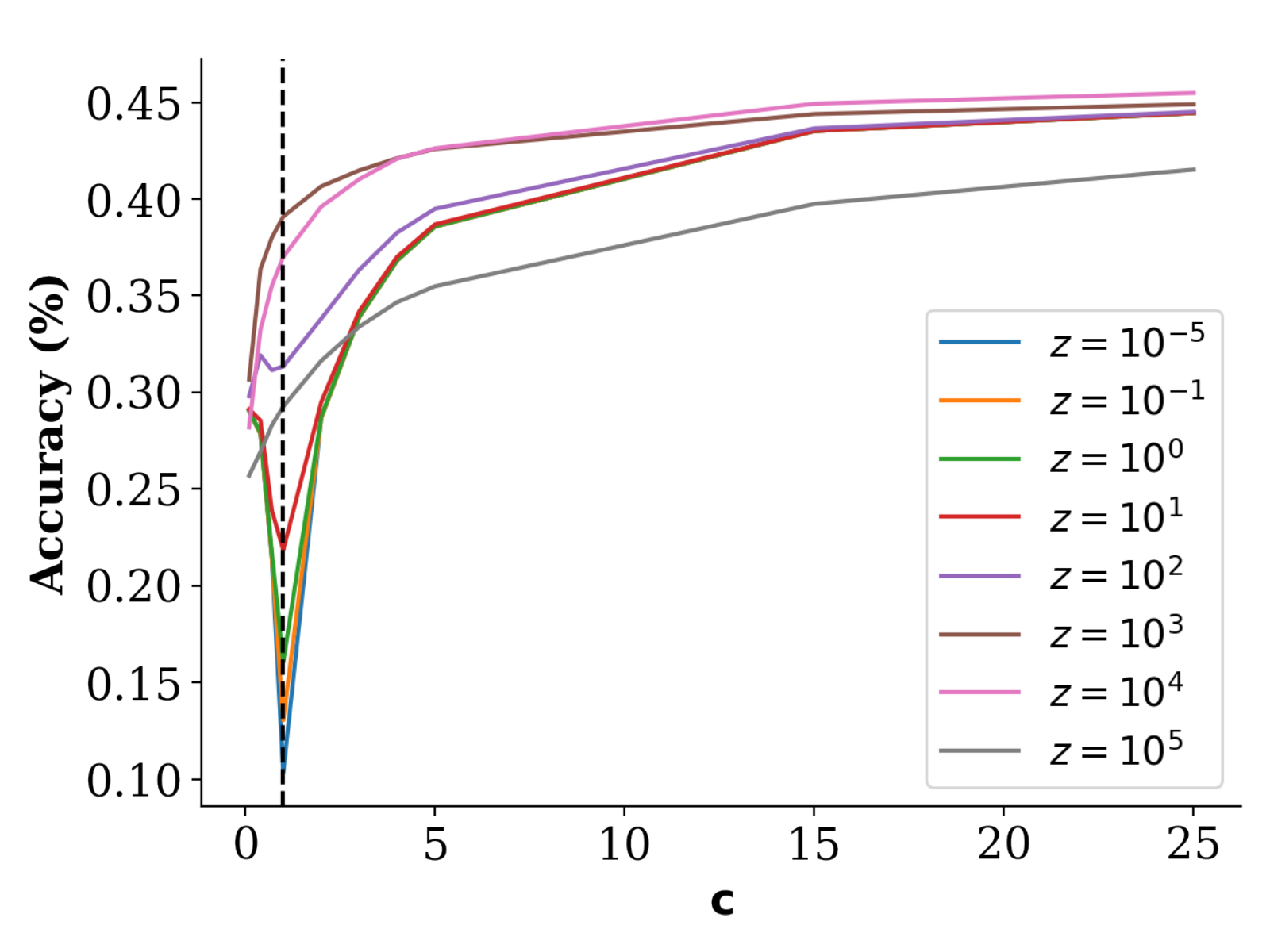}
        \caption{$n=640$}
        \label{fig:small_cifar_640_zoom}
    \end{subfigure}
    \hfill
    \begin{subfigure}[b]{0.24\textwidth}
        \centering
        \includegraphics[width=\textwidth]{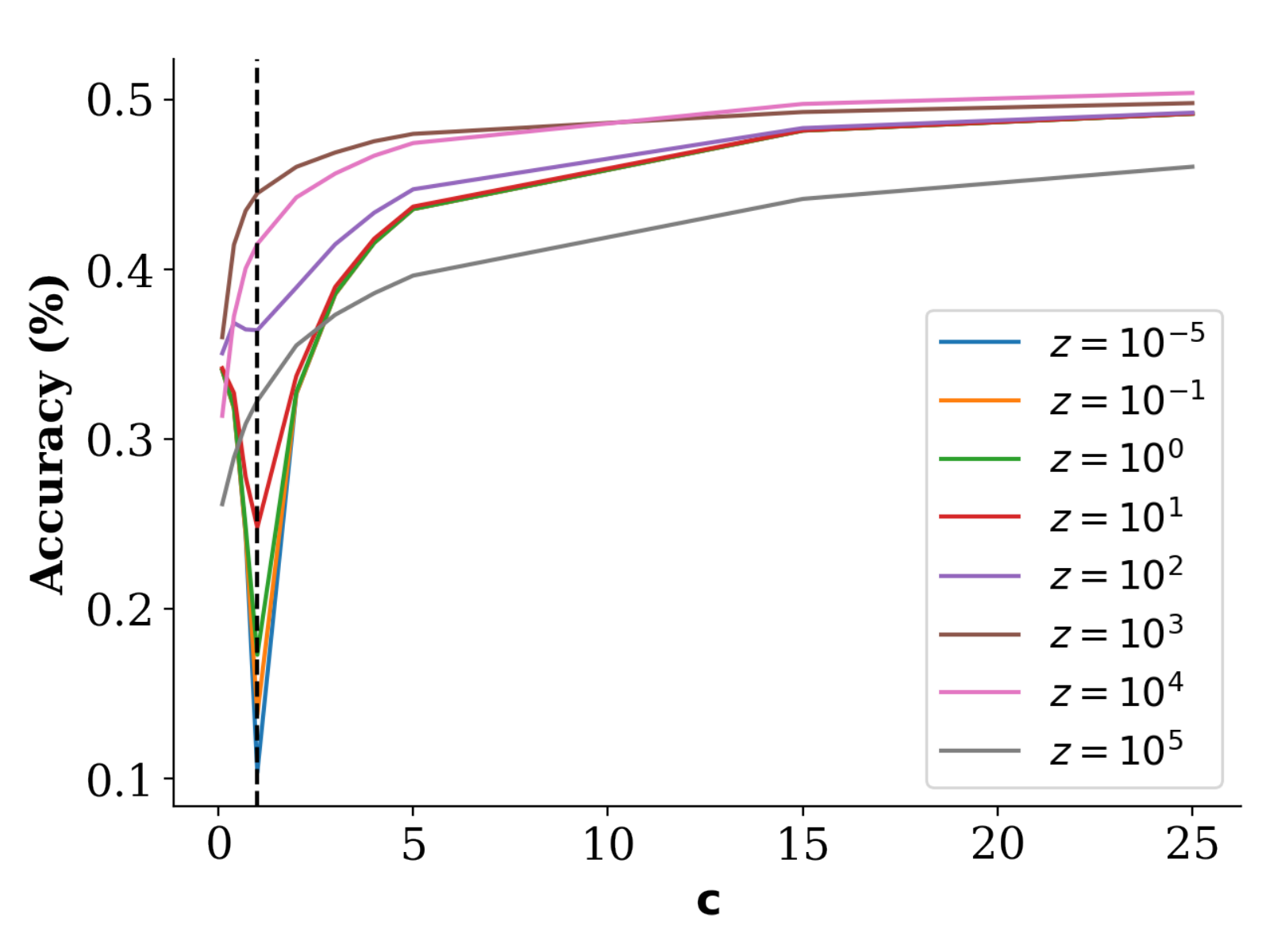}
        \caption{$n=1280$}
        \label{fig:small_cifar_1280_zoom}
    \end{subfigure}
    \caption{The figures above show FABReg's test accuracy increases with the model's complexity $c$ on different ($n$) subsampled  CIFAR-10 datasets. The expanded dataset follows similar patterns. We truncate the curve for $c > 25$ to better show the double descent phenomena. The full curves are shown in Appendix \ref{sec:additional_results}. Notice that when the shrinkage is sufficiently high, the double descent disappears, and the accuracy monotonically increases in complexity. Following \cite{kelly2022virtue}, we name this phenomenon {\it the virtue of complexity} (VoC). The test accuracy is averaged over 20 independent runs.} 
    \label{fig:small_cifar_voc_curve_zoom}
\end{figure*}

\begin{figure*}[]
    \centering
    \begin{subfigure}{0.45\textwidth}
        \centering
        \includegraphics[width=\textwidth]{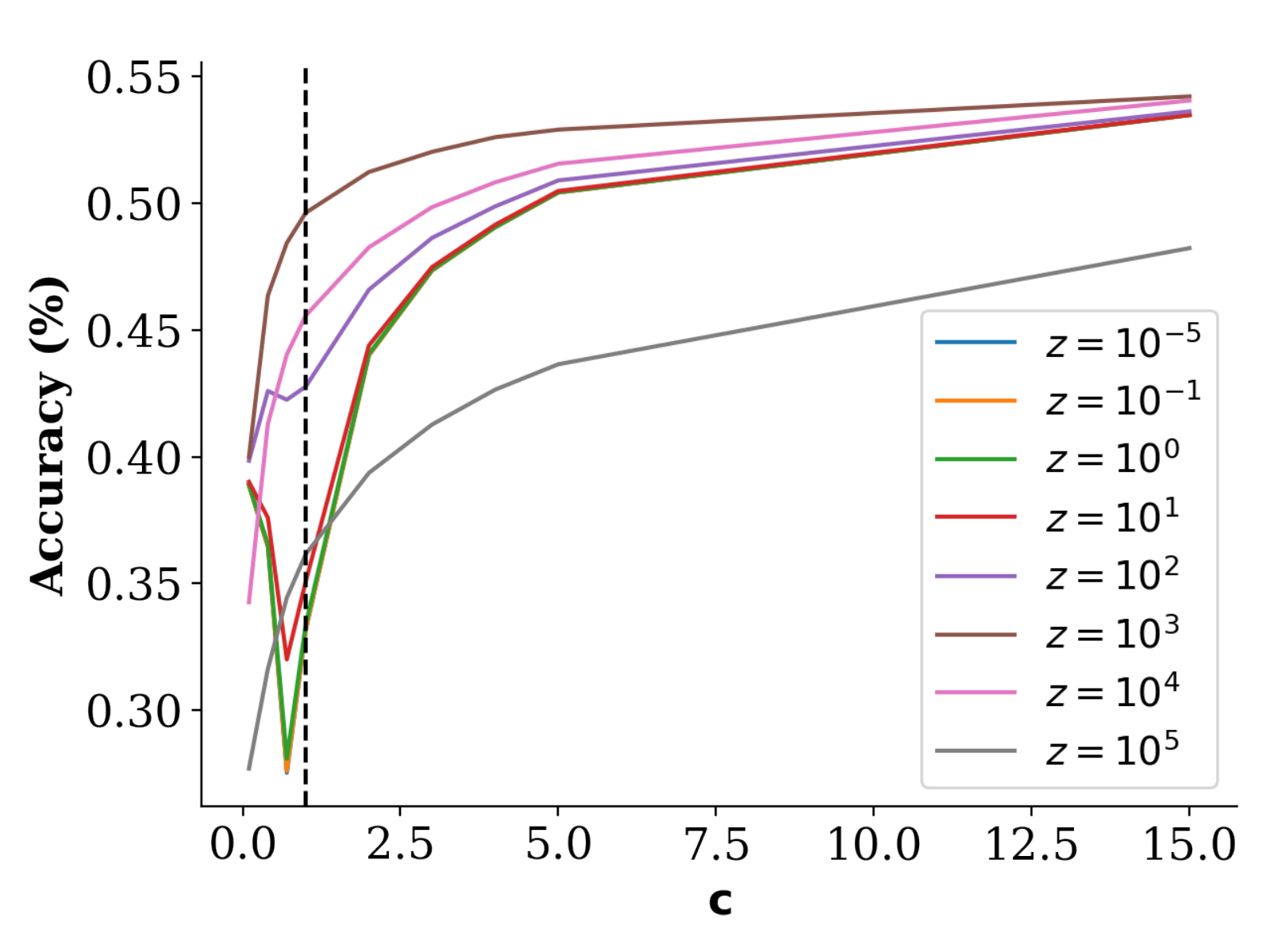}
        \caption{$n=2560$}
        \label{fig:nu_256}
    \end{subfigure}
    \begin{subfigure}{0.45\textwidth}
        \centering
        \includegraphics[width=\textwidth]{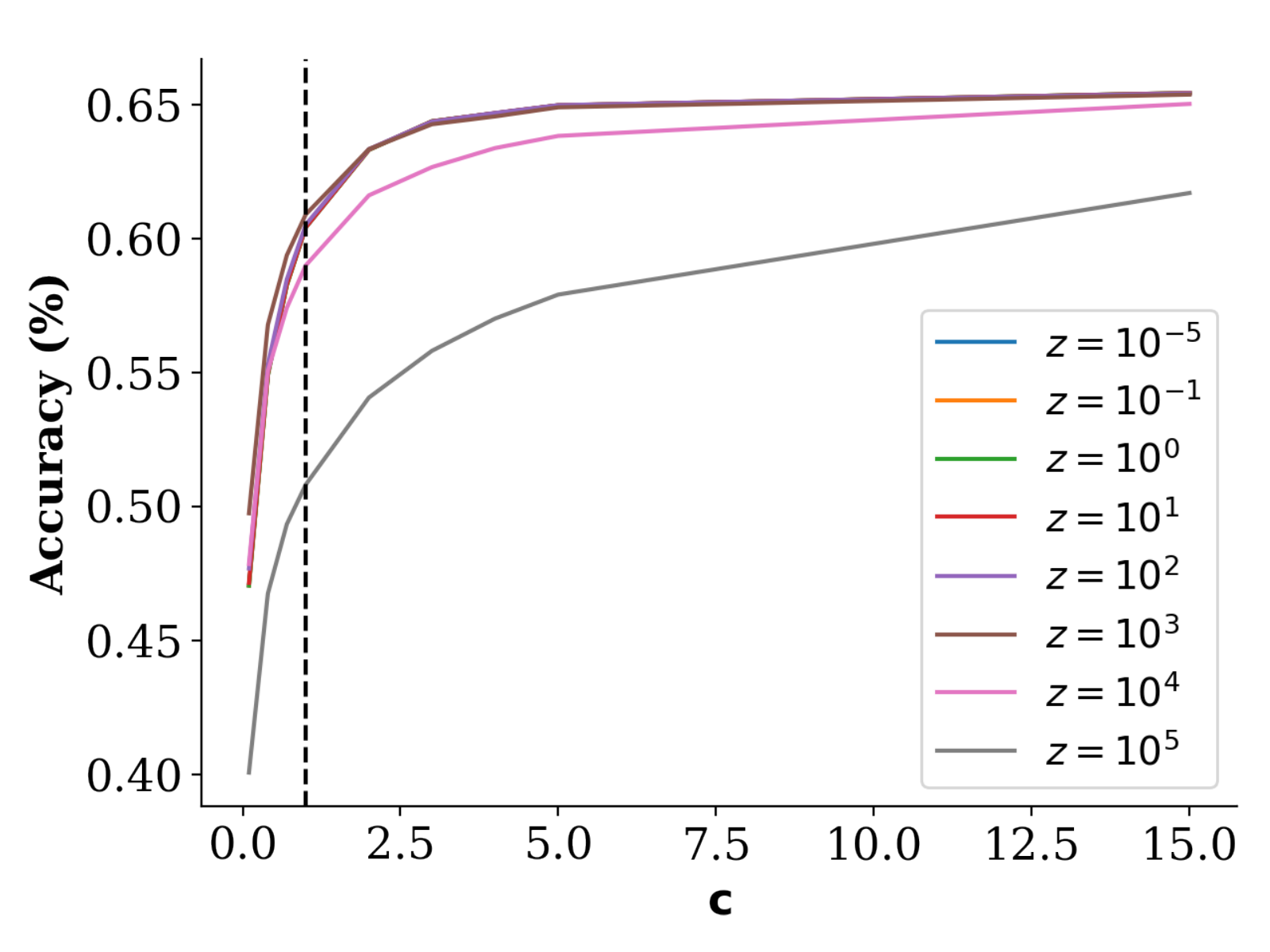}
        \caption{$n=50000$}
        \label{fig:nu_50000}
    \end{subfigure}
    
    \caption{The figures above show FABReg's test accuracy increases with the model's complexity $c$ on the subsampled CIFAR-10 dataset~\ref{fig:nu_256} and the full CIFAR-10 dataset~\ref{fig:nu_50000}. FABReg is trained using a $\nu=2000$ low-rank covariance matrix approximation. Notice that we still observe a (shifted) double descent when $\nu \approx n.$ The same phenomenon disappears when $\nu \ll n.$ The test accuracy is averaged over five independent runs.}
    \label{fig:nu_voc_curve}
\end{figure*}

\subsubsection{Big Datasets}
\label{sec:big_datasets}

In this section, we repeat the same experiments described in Section~\ref{sec:small_dataset}, but we extend the training set size $n$ up to the full CIFAR-10 dataset. For each $n$, we train FABReg, FABReg-$\nu$ with a rank-$\nu$ approximation as described in Algorithm~\ref{alg:spectral}, and the min-batch FABReg. We use $\nu=2000$ and $\text{batch size}=2000$ in the last two algorithms. Following  \cite{arora2019harnessing}, we train ResNet-34 as the benchmark for 160 epochs, with an initial learning rate of 0.001 and a batch size of 32. We decrease the learning rate by ten at epochs 80 and 120. ResNet-34 always reaches close to perfect accuracy on the training set, i.e., above 99\%. We run each training five times and report mean out-of-sample performance and its standard deviation. As the training sample is sufficiently large already, we set the model complexity to {\it only} $c=15$, meaning that for the full sample, FABReg performs a random feature ridge regression with $P=7.5 \times 10^{5}$. We report the results in Tables~\ref{sec:big_datasets} and \ref{table:batch_and_nu}. 
    
\input{big_dataset_cifar.tex}

\begin{table*}[t]
    \centering
    \caption{The table below shows the average test accuracy and standard deviation of FABReg-$\nu$ and mini-batch FABReg on the subsampled and full CIFAR-10 dataset. The test accuracy is average over five independent runs.}
    \scalebox{0.6}{
    \begin{tabular}{ccccccccc}
        \toprule
        \multirow{2}{*}{} &
        \multicolumn{2}{c}{$z=1$} &
        \multicolumn{2}{c}{$z=100$} &
        \multicolumn{2}{c}{$z=10000$} &
        \multicolumn{2}{c}{$z=100000$} \\
        \cmidrule{2-3} \cmidrule{4-5}  \cmidrule{6-7}  \cmidrule{8-9}

      FABReg & batch $= 2000$ &         $\nu=2000$  &      batch $= 2000$  &  $\nu=2000$  &  batch $= 2000$ &     $\nu=2000$ & batch $= 2000$ &    $\nu=2000$ \\
      n & & & & & & & & \\
    \midrule
    2560 & 53.13\% $\pm$ 0.38\% & 53.48\% $\pm$ 0.22\% & 53.15\% $\pm$ 0.42\% & 53.63\% $\pm$ 0.24\% & 52.01\% $\pm$ 0.51\% & 54.05\% $\pm$ 0.44\% & 46.78\% $\pm$ 0.52\% & 48.23\% $\pm$ 0.34\% \\
    5120 & 57.68\% $\pm$ 0.18\% & 57.63\% $\pm$ 0.19\% &  57.70\% $\pm$ 0.16\% & 57.63\% $\pm$ 0.18\% & 56.83\% $\pm$ 0.27\% & 57.53\% $\pm$ 0.11\% & 51.42\% $\pm$ 0.22\% & 51.75\% $\pm$ 0.14\% \\
    10240 & 59.79\% $\pm$ 0.35\% &  61.20\% $\pm$ 0.39\% & 59.79\% $\pm$ 0.35\% &  61.20\% $\pm$ 0.38\% & 58.63\% $\pm$ 0.28\% & 60.63\% $\pm$ 0.21\% & 53.73\% $\pm$ 0.37\% & 55.16\% $\pm$ 0.34\% \\
    20480 &  61.56\% $\pm$ 0.35\% &  63.50\% $\pm$ 0.12\% & 61.55\% $\pm$ 0.37\% &  63.50\% $\pm$ 0.13\% &   60.90\% $\pm$ 0.20\% & 62.92\% $\pm$ 0.12\% &  57.10\% $\pm$ 0.19\% &  58.40\% $\pm$ 0.21\% \\

50000 &  62.74\% $\pm$ 0.10\% & 65.45\% $\pm$ 0.18\% &  62.74\% $\pm$ 0.10\% & 65.44\% $\pm$ 0.18\% & 62.35\% $\pm$ 0.05\% & 65.04\% $\pm$ 0.19\% & 59.99\% $\pm$ 0.02\% & 61.71\% $\pm$ 0.09\% \\
    \bottomrule
    \end{tabular}
    }

    \label{table:batch_and_nu}
\end{table*}
The experiment delivers a number of additional conclusions:
\begin{itemize}
    \item First, we observe that, while for small train sample sizes of $n \le 10000$, simple kernel methods achieve performance comparable with that of DNNs, this is not the case for $n > 20000.$ Beating DNNs on big datasets with  shallow methods requires more complex kernels, such as those in \cite{shankar2020neural,li2019enhanced}.
    
    \item Second, we confirm the findings of \cite{ma2017diving, lee2020finite} suggesting that the role of small 
    
    eigenvalues is important. For example, FABReg-$\nu$ with $\nu=2000$ loses several percent of accuracy on larger datasets.
    
    \item Third, surprisingly, both the mini-batch FABReg and FABReg-$\nu$ sometimes achieve higher accuracy than the full sample regression on moderately-sized datasets. See Tables~\ref{table:big_dataset_cifar} and \ref{table:batch_and_nu}. Understanding these phenomena is an interesting direction for future research. 
    
    \item Fourth, the double descent phenomenon naturally appears for both FABReg-$\nu$ and the mini-batch FABReg but only when $\nu \approx n$ or $\text{batch size} \approx n$. However, the double descent phenomenon disappears when $\nu \ll n$. This intriguing finding is shown in Figure \ref{fig:nu_voc_curve} for FABReg-$\nu$, and in Appendix \ref{sec:additional_results} for the mini-batch FABReg.
    
    \item Fifth, on average, FABReg-$\nu$ outperforms mini-batch FABReg on larger datasets.
\end{itemize}

\section{Conclusion and Discussion}
\label{sec:conclusion}
The recent discovery of the equivalence between infinitely wide neural networks (NNs) in the lazy training regime and neural tangent kernels (NTKs) \cite{jacot2018neural} has revived interest in kernel methods. However, these kernels are extremely complex and usually require running on big and expensive computing clusters \cite{avron2017faster, shankar2020neural} due to memory (RAM) requirements. This paper proposes a highly scalable random features ridge regression that can run on a simple laptop. We name it Fast Annihilating Batch Regression (FABReg). Thanks to the linear algebraic properties of covariance matrices, this tool can be applied to any kernel and any way of generating random features. Moreover, we provide several experimental results to assess its performance. We show how FABReg can outperform (in training and prediction speed) the current state-of-the-art ridge classifier's implementation. Then, we show how a simple data representation strategy combined with a random features ridge regression can outperform complicated kernels (CNTKs) and over-parametrized Deep Neural Networks (ResNet-34) in the few-shot learning setting. The experiments section concludes by showing additional results on big datasets. In this paper, we focus on very simple classes of random features. Recent findings (see, e.g., \cite{shankar2020neural}) suggest that highly complex kernel architectures are necessary to achieve competitive performance on large datasets. Since each kernel regression can be approximated with random features, our method is potentially applicable to these kernels as well. However, directly computing the random feature representation of such complex kernels is non-trivial and we leave it for future research. 

\clearpage

\bibliographystyle{plainnat}
\bibliography{references}

\newpage
\appendix
\onecolumn
\section{Proofs}

\begin{proof}[Proof of Lemma \ref{bound}] We have 
\begin{equation}
\begin{aligned}
&\Psi_{k+1}\ =\ \Psi_k\ +\ S_{k+1}S_{k+1}'\\
&\tilde\Psi_{k+1}\ =\ \hat\Psi_k\ +\ S_{k+1}S_{k+1}'\\
&\hat\Psi_{k+1}\ =\ P_{k+1}\tilde\Psi_{k+1}P_{k+1}\,.
\end{aligned}
\end{equation}
By the definition of the spectral projection, we have 
\begin{equation}
\|\tilde\Psi_{k+1}-\hat\Psi_{k+1}\|\ \le\ \gl_{\nu+1}(\tilde\Psi_{k+1})\ \le\ 
\gl_{\nu+1}(\Psi_{k+1})\,, 
\end{equation}
and hence 
\begin{equation}
\begin{aligned}
&\|\Psi_{k+1}-\hat\Psi_{k+1}\|\\ 
&\le\ \|\Psi_{k+1}-\tilde\Psi_{k+1}\|+\|\tilde\Psi_{k+1}-\hat\Psi_{k+1}\|\\
&=\ \|\Psi_{k}-\hat\Psi_{k}\|+\|\tilde\Psi_{k+1}-\hat\Psi_{k+1}\|)\\
&\le\ \|\Psi_{k}-\hat\Psi_{k}\|\ +\ \gl_{\nu+1}(\Psi_{k+1})\,,
\end{aligned}
\end{equation}
and the claim follows by induction. The last claim follows from the simple inequality 
\begin{equation}
\|(\Psi_{k+1}+zI)^{-1}-(\hat\Psi_{k+1}+zI)^{-1}\|\ \le\ z^{-2}\|\Psi_{k+1}-\hat\Psi_{k+1}\|\,.
\end{equation}
\end{proof}

\section{Additional Experimental Results}
\label{sec:additional_results}
This section provides additional experiments and findings that may help the community with future research.

First, we dive into more details about our comparison with {\it sklearn}. Table~\ref{table:sklearn_vs_FABReg} shows a more detailed training and prediction time comparison between FABReg and {\it sklearn}. In particular, we average training and prediction time over five independent runs. The experiment settings are explained in Section~\ref{sec:sklearn}. We show how one, depending on the number shrinkages $|z|$, would start considering using FABReg when the number of observations in the dataset $n \approx 5000$. In this case, we have used the {\it numpy} linear algebra library to decompose FABReg's covariance matrix, which appears to be faster than the {\it scipy} counterpart. We share our code in the following repository: \url{https://github.com/tengandreaxu/fabr}. 

\begin{table}[b]
\centering
\caption{The table below shows FABReg and {\it sklearn}'s training and prediction time (in seconds) on a synthetic dataset. We vary the dataset number of features $d$ and the number of shrinkages ($|z|$). We report the average running time and the standard deviation over five independent runs.}
\scalebox{0.6}{
\begin{tabular}{ccccccccc}
\toprule

  \multirow{2}{*}{} &
  \multicolumn{2}{c}{$|z|=5$} &
  \multicolumn{2}{c}{$|z|=10$} &
  \multicolumn{2}{c}{$|z|=20$} &
  \multicolumn{2}{c}{$|z|=50$}\\
  \cmidrule{2-3} \cmidrule{4-5}  \cmidrule{6-7}  \cmidrule{8-9}
&  FABReg & ${\it sklearn}$ &  FABReg & ${\it sklearn}$ &  FABReg & ${\it sklearn}$ & FABReg & ${\it sklearn}$ \\
$d$    &                &                           &                 &                            &                 &                            &                 &                            \\
\midrule
10     &   7.72s $\pm$ 0.36s &          \textbf{0.01s $\pm$ 0.00s} &    6.90s $\pm$ 0.77s &           \textbf{0.02s $\pm$ 0.00s} &   7.04s $\pm$ 0.67s &           \textbf{0.03s $\pm$ 0.00s} &   7.44s $\pm$ 0.57s &         \textbf{ 0.07s $\pm$ 0.01s} \\
100    &   7.35s $\pm$ 0.36s &         \textbf{0.06s $\pm$ 0.02s} &   6.58s $\pm$ 0.34s &          \textbf{0.11s $\pm$ 0.01s} &   7.61s $\pm$ 1.14s &          \textbf{0.24s $\pm$ 0.04s} &    7.3s $\pm$ 0.49s &          \textbf{0.53s $\pm$ 0.06s} \\
500    &   7.37s $\pm$ 0.44s &         \textbf{0.33s $\pm$ 0.16s} &   6.81s $\pm$ 0.25s &          \textbf{0.54s $\pm$ 0.03s} &   7.02s $\pm$ 0.35s &          \textbf{1.01s $\pm$ 0.07s} &   7.44s $\pm$ 0.48s &          \textbf{2.41s $\pm$ 0.21s} \\
1000   &   7.62s $\pm$ 0.31s &         \textbf{0.58s $\pm$ 0.21s} &   7.38s $\pm$ 0.23s &          \textbf{1.06s $\pm$ 0.04s} &   7.51s $\pm$ 0.24s &          \textbf{2.04s $\pm$ 0.04s} &   7.69s $\pm$ 0.08s &          \textbf{4.79s $\pm$ 0.36s} \\
2000   &   8.33s $\pm$ 0.42s &         \textbf{1.21s $\pm$ 0.03s} &   8.09s $\pm$ 0.73s &          \textbf{2.44s $\pm$ 0.05s} &   8.33s $\pm$ 0.24s &          \textbf{4.87s $\pm$ 0.07s} &   \textbf{8.29s $\pm$ 0.47s} &         12.21s $\pm$ 0.15s \\
3000   &   9.24s $\pm$ 0.25s &         \textbf{2.49s $\pm$ 0.05s} &   9.18s $\pm$ 0.41s &          \textbf{5.08s $\pm$ 0.03s} &    \textbf{9.51s $\pm$ 0.20s} &         10.06s $\pm$ 0.02s &   \textbf{9.67s $\pm$ 0.41s} &         25.67s $\pm$ 0.23s \\
5000   &  10.64s $\pm$ 0.86s &         \textbf{5.36s $\pm$ 0.05s} &   11.01s $\pm$ 0.7s &         \textbf{10.74s $\pm$ 0.06s} &  \textbf{11.57s $\pm$ 0.81s} &         21.31s $\pm$ 0.12s &  \textbf{11.54s $\pm$ 0.41s} &         54.18s $\pm$ 0.73s \\
10000  &  \textbf{11.49s $\pm$ 0.66s} &        17.87s $\pm$ 8.58s &  \textbf{11.81s $\pm$ 0.47s} &        28.32s $\pm$ 10.53s &  \textbf{11.61s $\pm$ 0.49s} &         44.72s $\pm$ 9.99s &   \textbf{12.55s $\pm$ 0.3s} &       101.58s $\pm$ 15.66s \\
25000  &  \textbf{13.89s $\pm$ 0.21s} &        27.79s $\pm$ 8.75s &   \textbf{14.50s $\pm$ 0.45s} &         49.84s $\pm$ 9.68s &  \textbf{14.46s $\pm$ 0.96s} &        94.08s $\pm$ 10.94s &  \textbf{15.68s $\pm$ 0.74s} &       224.31s $\pm$ 11.75s \\
50000  &  \textbf{17.99s $\pm$ 0.22s} &        50.51s $\pm$ 8.99s &  \textbf{18.27s $\pm$ 0.37s} &        92.88s $\pm$ 10.45s &   \textbf{19.10s $\pm$ 0.37s} &       176.24s $\pm$ 10.07s &  \textbf{19.68s $\pm$ 0.85s} &       422.95s $\pm$ 13.22s \\
100000 &   \textbf{25.30s $\pm$ 0.39s} &        95.57s $\pm$ 0.25s &  \textbf{26.16s $\pm$ 0.46s} &        177.54s $\pm$ 3.77s &  \textbf{27.93s $\pm$ 0.35s} &        340.32s $\pm$ 3.74s &  \textbf{29.48s $\pm$ 1.38s} &        816.25s $\pm$ 4.35s \\
\bottomrule
\label{table:sklearn_vs_FABReg}
\end{tabular}

}
\end{table}

Second, while Figure~\ref{fig:small_cifar_voc_curve_zoom} shows FABReg's test accuracy on increasing complexity $c$ truncated curves, we present here the whole picture; i.e., Figure~\ref{fig:small_cifar_voc_curve} shows full FABReg's test accuracy increases with the model's complexity $c$ on different ($n$) subsampled CIFAR-10 datasets averaged over twenty independent runs. The expanded dataset follows similar patterns. Similar to Figure~\ref{fig:small_cifar_voc_curve_zoom}, one can notice that when the shrinkage is sufficiently high, the double descent disappears, and the accuracy monotonically increases in complexity.

Third, the double descent phenomenon naturally appears for both FABReg-$\nu$ and the mini-batch FABReg but only when $\nu \approx n$ or $\text{batch size} \approx n$. However, the double descent phenomenon disappears when $\nu \ll n$. This intriguing finding is shown in Figure \ref{fig:nu_voc_curve} for FABReg-$\nu$, and here, in Figure~\ref{fig:batch_voc_curve}, we report the same curves for mini-batch FABReg.

\begin{figure*}[t]
\captionsetup[subfigure]{justification=centering}
    \centering
    \begin{subfigure}[b]{0.24\textwidth}
        \centering
        \includegraphics[width=\textwidth]{1_voc_curve.pdf}
        \caption{$n=10$}
        \label{fig:small_cifar_10}
    \end{subfigure}
    \hfill
    \begin{subfigure}[b]{0.24\textwidth}
        \centering
        \includegraphics[width=\textwidth]{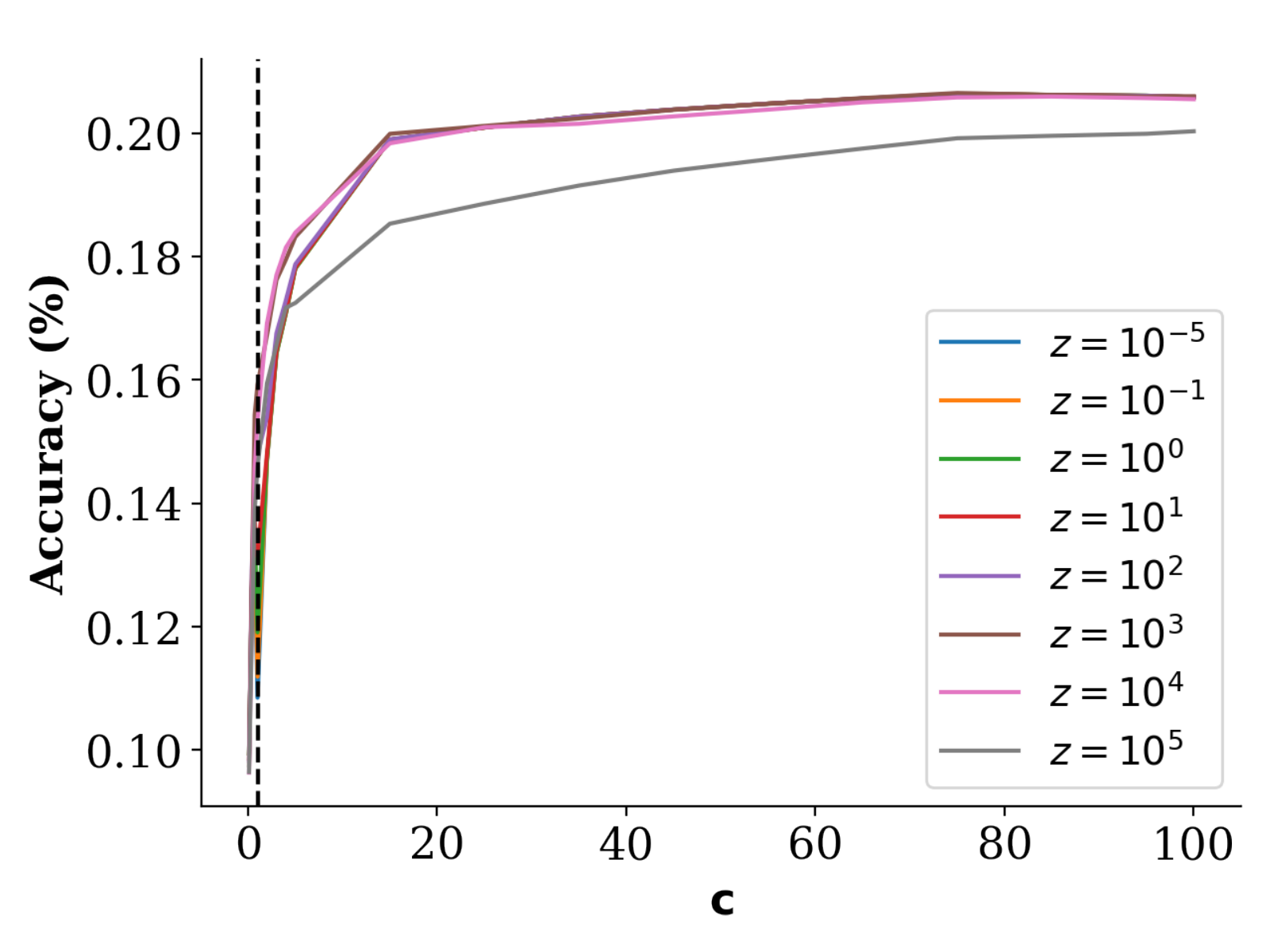}
        \caption{$n=20$}
        \label{fig:small_cifar_20}
    \end{subfigure}
    \hfill
    \begin{subfigure}[b]{0.24\textwidth}
        \centering
        \includegraphics[width=\textwidth]{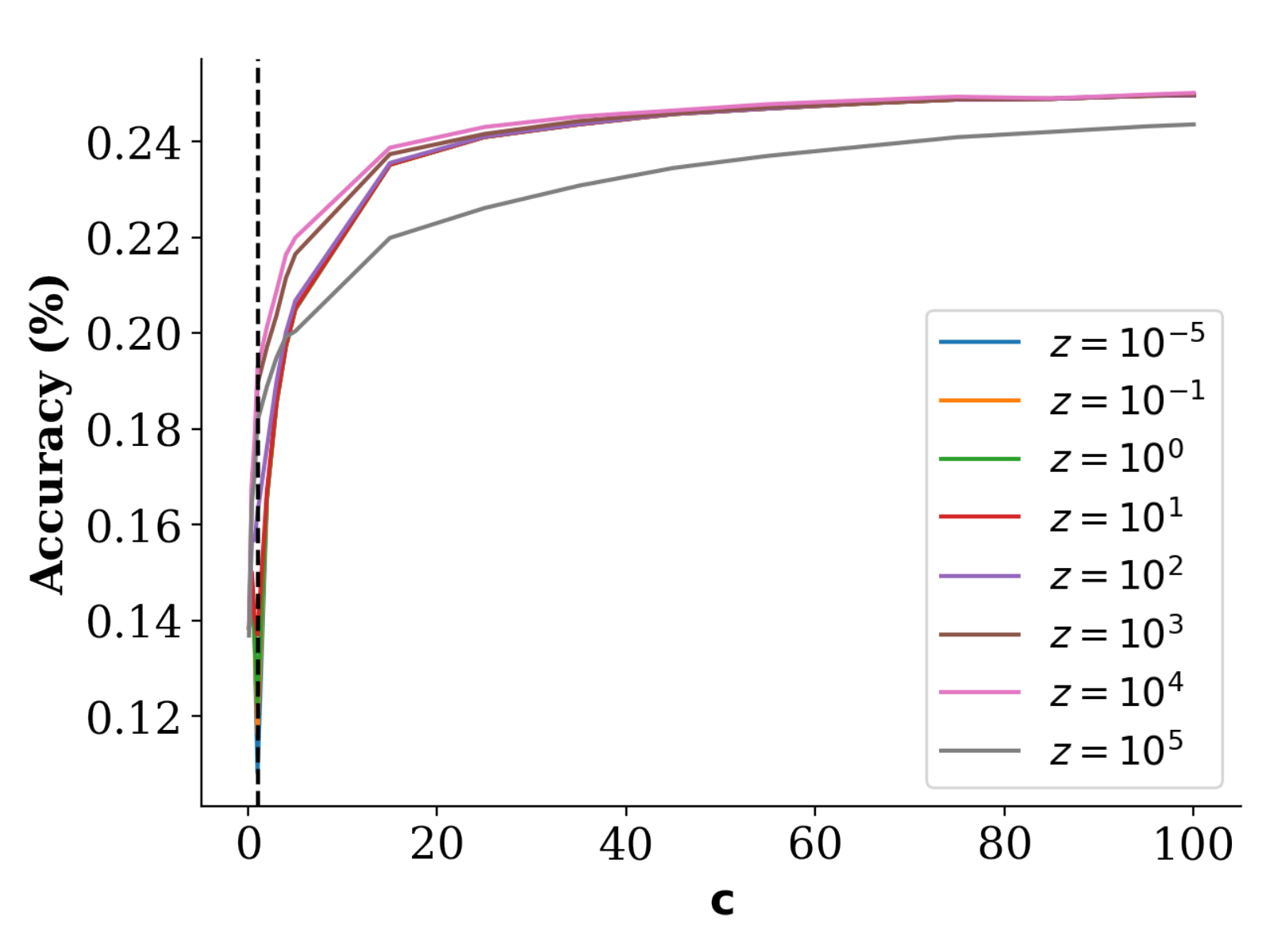}
        \caption{$n=40$}
        \label{fig:small_cifar_40}
    \end{subfigure}
    \hfill
    \begin{subfigure}[b]{0.24\textwidth}
        \centering
        \includegraphics[width=\textwidth]{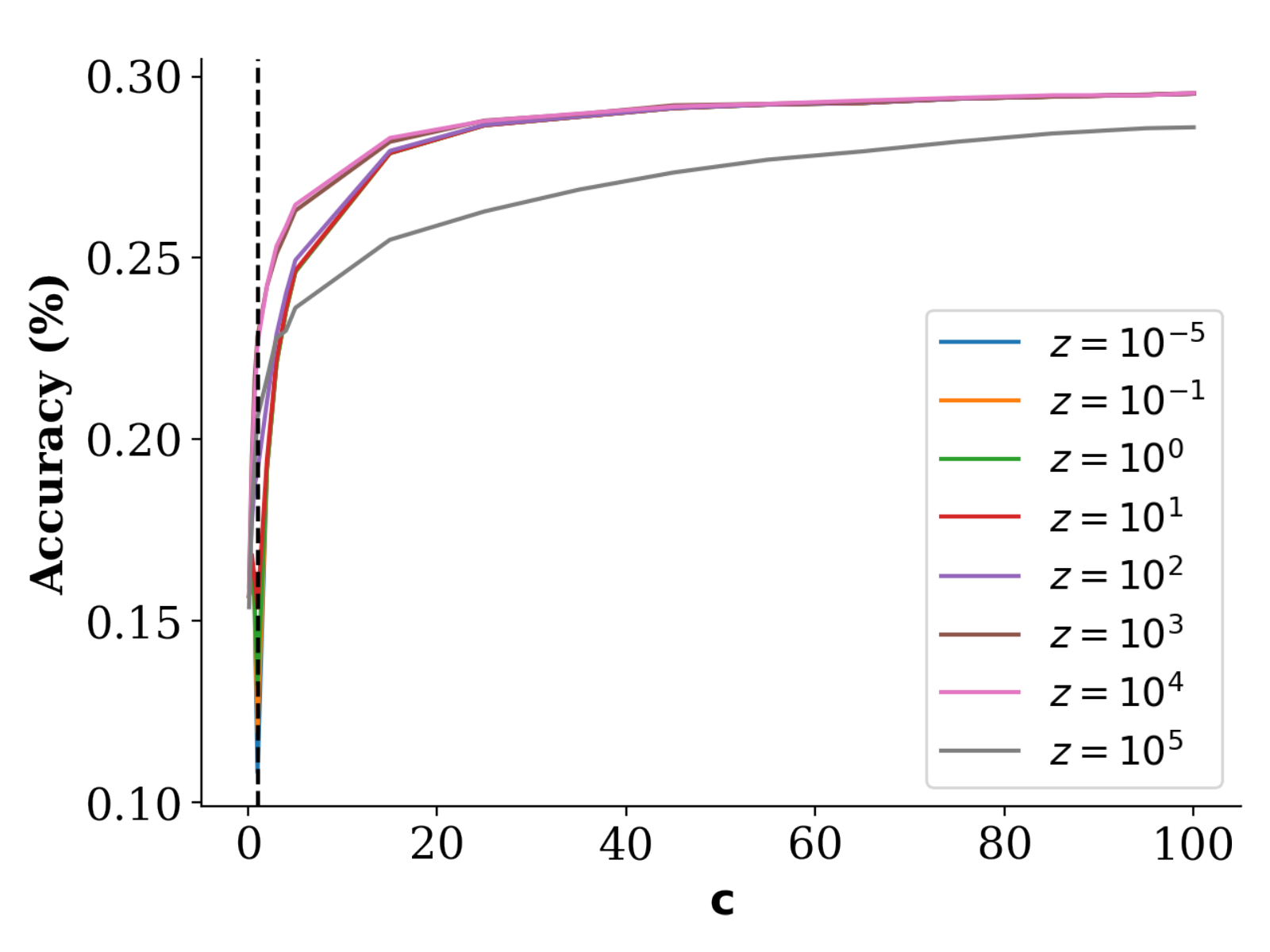}
        \caption{$n=80$}
        \label{fig:small_cifar_80}
    \end{subfigure}
    \begin{subfigure}[b]{0.24\textwidth}
        \centering
        \includegraphics[width=\textwidth]{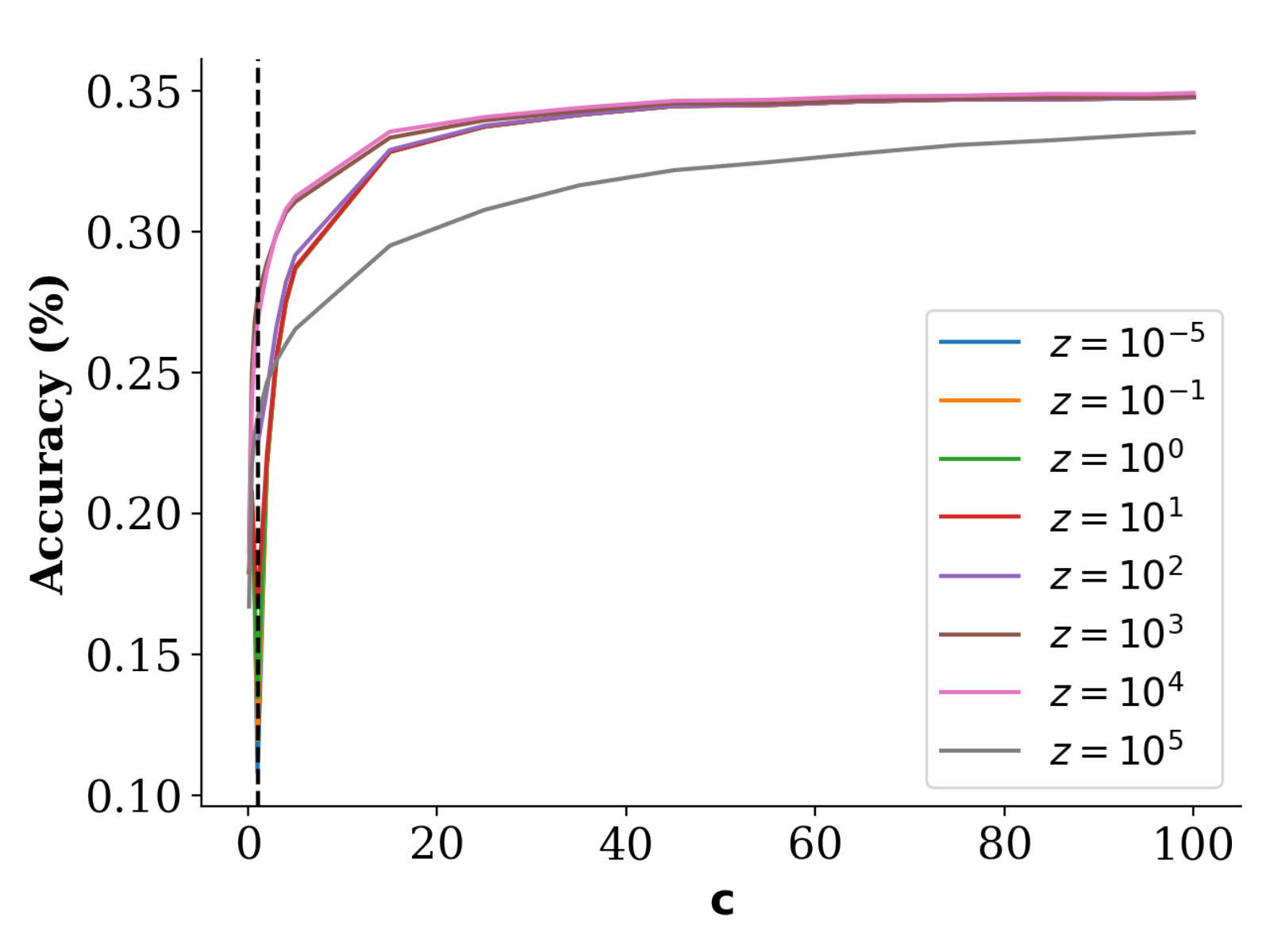}
        \caption{$n=160$}
        \label{fig:small_cifar_160}
    \end{subfigure}
    \hfill
    \begin{subfigure}[b]{0.24\textwidth}
        \centering
        \includegraphics[width=\textwidth]{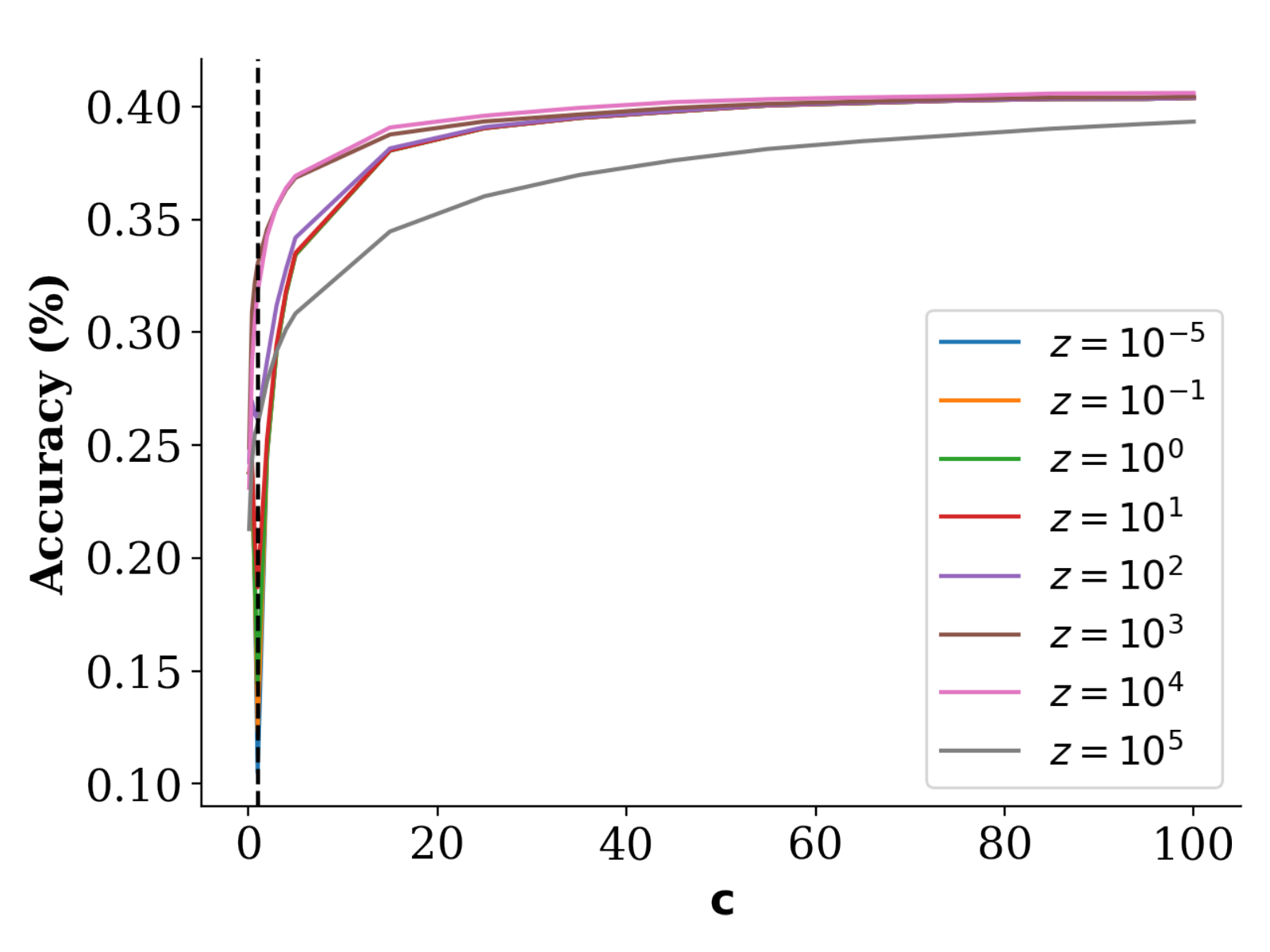}
        \caption{$n=320$}
        \label{fig:small_cifar_320}
    \end{subfigure}
    \hfill
    \begin{subfigure}[b]{0.24\textwidth}
        \centering
        \includegraphics[width=\textwidth]{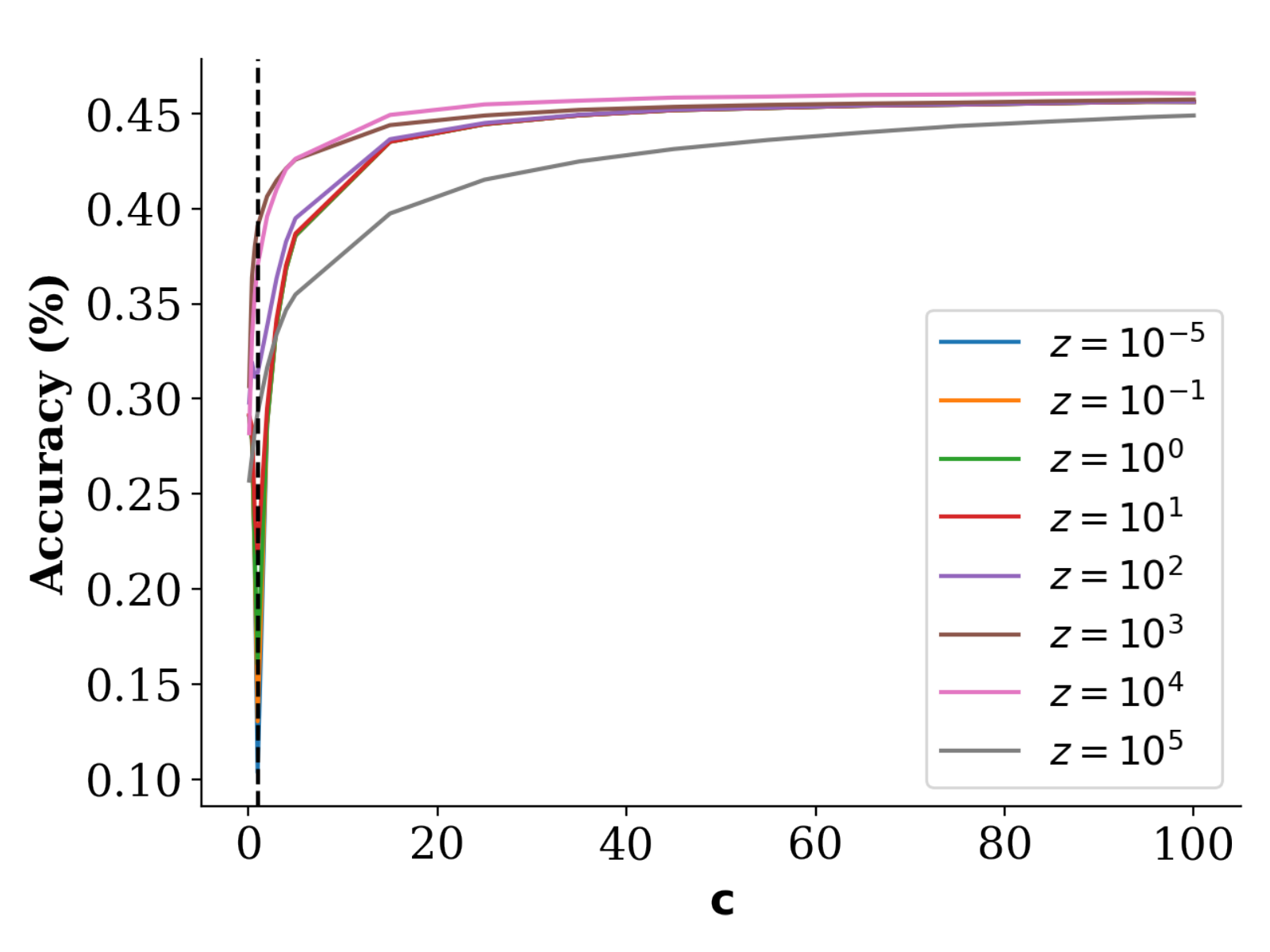}
        \caption{$n=640$}
        \label{fig:small_cifar_640}
    \end{subfigure}
    \hfill
    \begin{subfigure}[b]{0.24\textwidth}
        \centering
        \includegraphics[width=\textwidth]{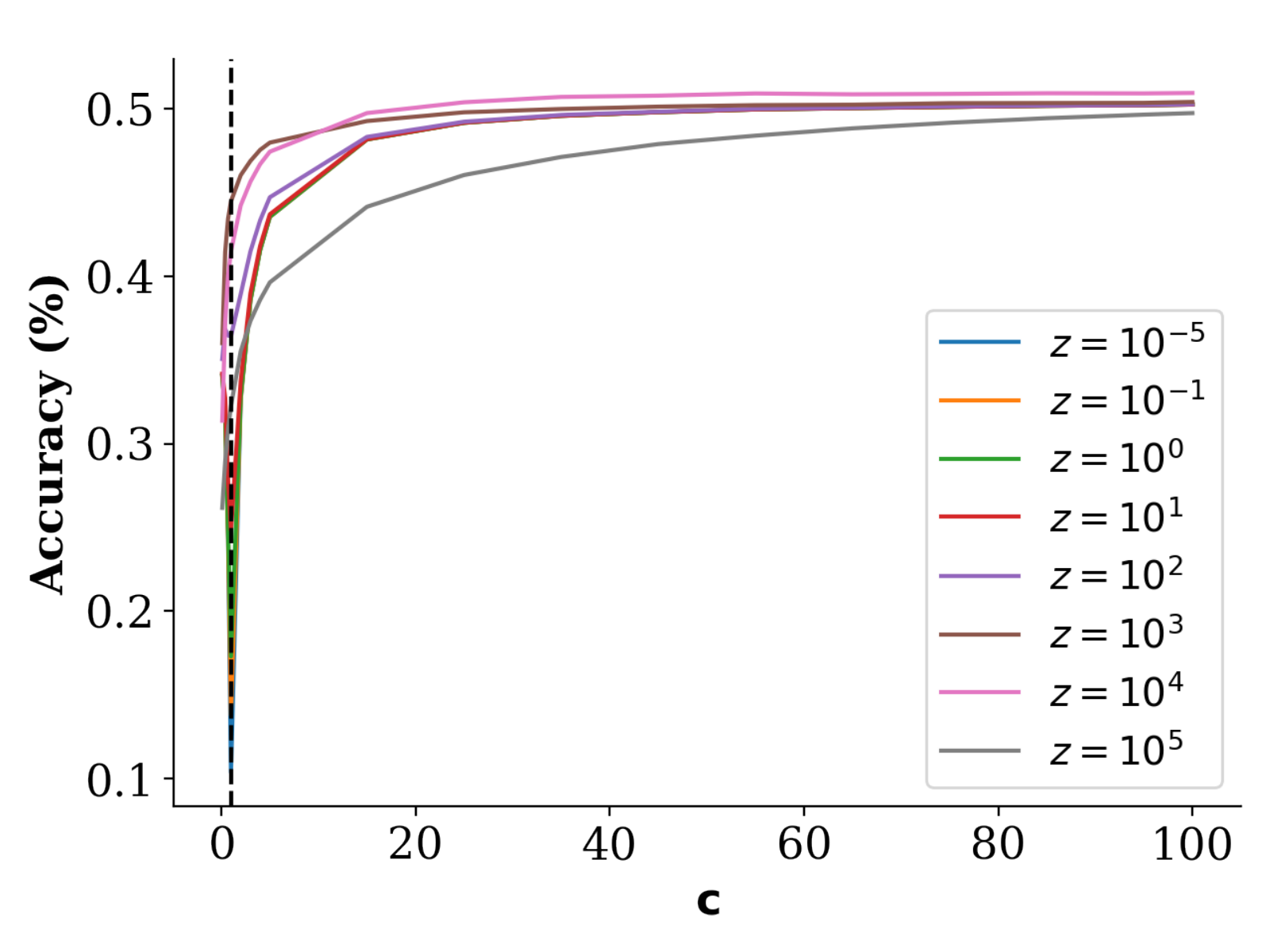}
        \caption{$n=1280$}
        \label{fig:small_cifar_1280}
    \end{subfigure}
    \caption{The figure above shows the full FABReg's accuracy increase with the model's complexity $c$ in the small dataset regime. The expanded dataset follows similar patterns.} 
    \label{fig:small_cifar_voc_curve}
\end{figure*}

\begin{figure*}[h!]
\captionsetup[subfigure]{justification=centering}
    \centering

    \begin{subfigure}[b]{0.45\textwidth}
        \centering
        \includegraphics[width=\textwidth]{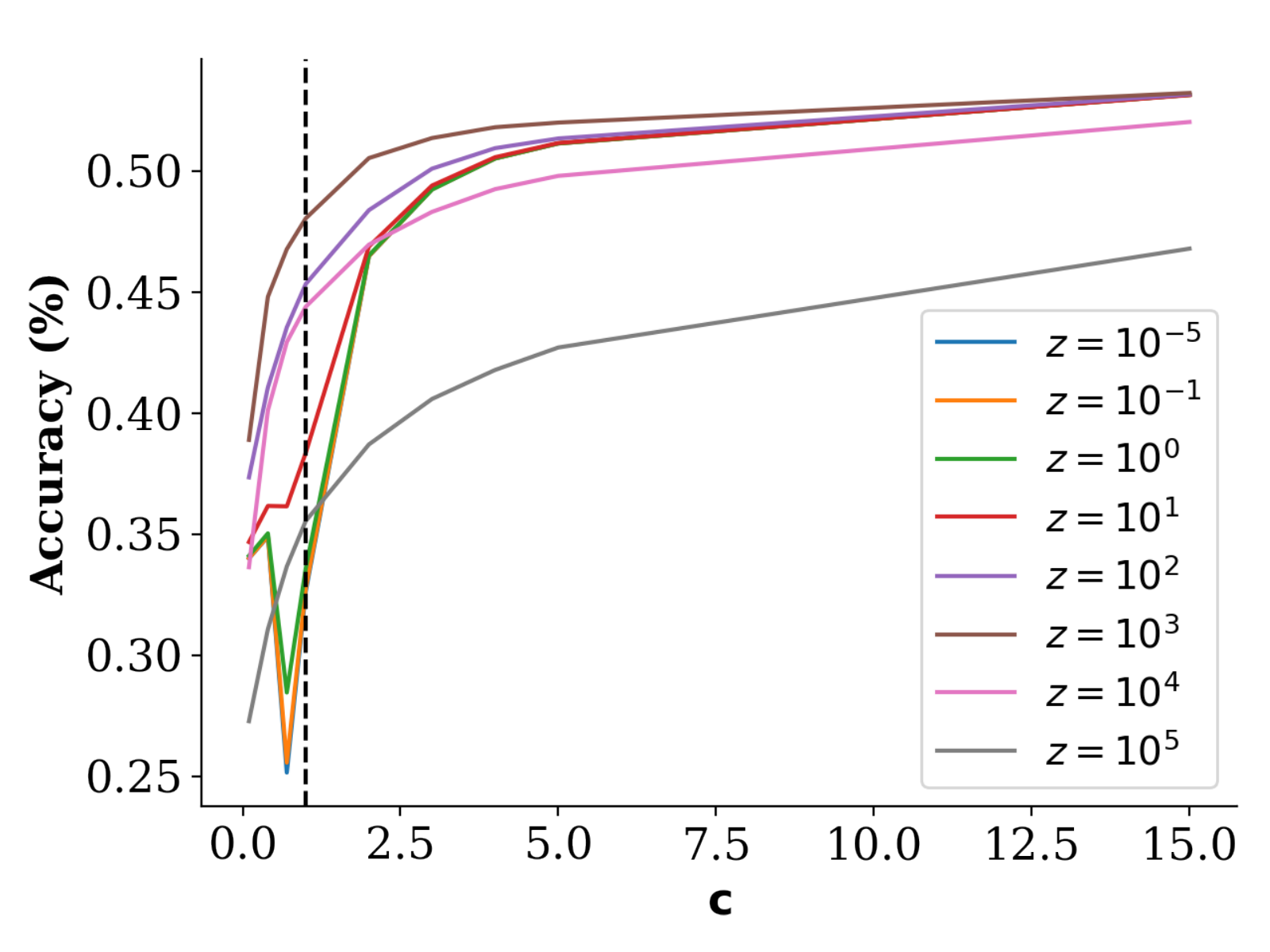}
        \caption{$n=2560$}
        \label{fig:batch_cifar_2560}
    \end{subfigure}
    \begin{subfigure}[b]{0.45\textwidth}
        \centering
        \includegraphics[width=\textwidth]{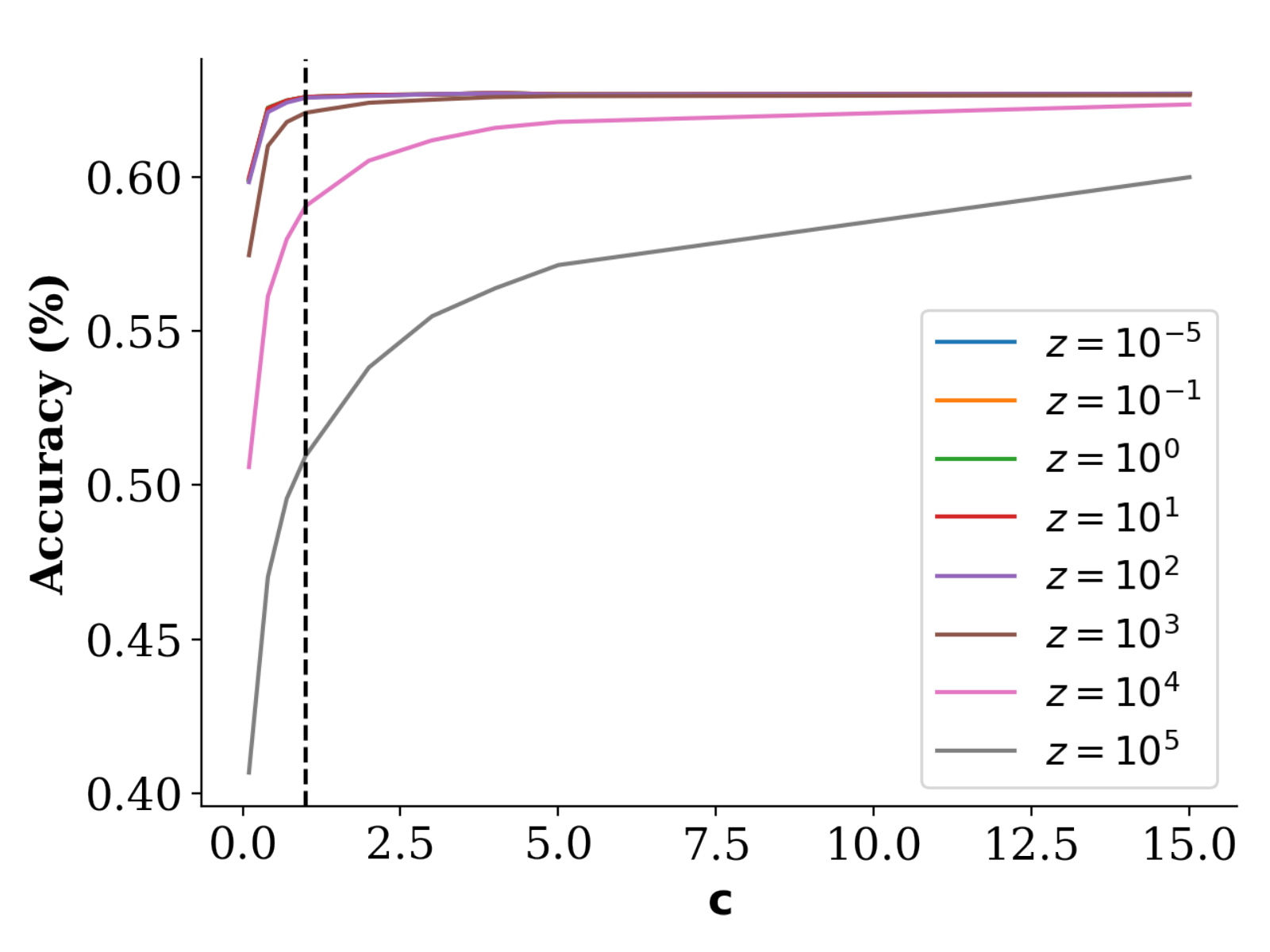}
        \caption{$n=50000$}
        \label{fig:batch_cifar_50000}
    \end{subfigure}
   
    \caption{Similar to Figure \ref{fig:nu_voc_curve}, the figures above show FABReg's test accuracy increases with the model's complexity $c$ on the subsampled CIFAR-10 dataset~\ref{fig:batch_cifar_2560} and the full CIFAR-10 dataset~\ref{fig:batch_cifar_50000}. FABReg trains using mini-batches with $\text{batch size=2000}$ in both cases. Notice that we still observe a (shifted) double descent when $\text{batch size} \approx n$, while the same phenomenon disappears when $\text{batch size} \ll n.$ The test accuracy is averaged over 5 independent runs.} 
    \label{fig:batch_voc_curve}
\end{figure*}
\clearpage

\end{document}

%% file: small_dataset_cifar.tex
\begin{table*}[h]
    \captionsetup{justification=centering}
    \caption{The table below shows the average test accuracy and standard deviation of ResNet-34, CNTK, and FABR on the subsampled CIFAR-10 datasets. The test accuracy is average over twenty independent runs.}
    
    \centering
    \label{table:small_dataset_cifar}

\scalebox{0.6}{

\begin{tabular}{rcccccc}
\toprule
   n &      ResNet-34 & 14-layer CNTK   &     z=1  &             z=100  &             z=10000 &           z=100000 \\
\midrule
  10 & 14.59\% $\pm$ 1.99\% & 15.33\% $\pm$ 2.43\% &  18.50\% $\pm$ 2.18\% &  \textbf{18.50\% $\pm$ 2.18\%} & 18.42\% $\pm$ 2.13\% & 18.13\% $\pm$ 2.01\% \\
  20 & 17.50\% $\pm$ 2.47\% & 18.79\% $\pm$ 2.13\% & 20.84\% $\pm$ 2.38\% & \textbf{20.85\% $\pm$ 2.38\%} & 20.78\% $\pm$ 2.35\% & 20.13\% $\pm$ 2.34\% \\
  40 & 19.52\% $\pm$ 1.39\% & 21.34\% $\pm$ 1.91\% & 25.09\% $\pm$ 1.76\% &  25.10\% $\pm$ 1.76\% & \textbf{25.14\% $\pm$ 1.75\%} & 24.41\% $\pm$ 1.88\% \\
  80 & 23.32\% $\pm$ 1.61\% & 25.48\% $\pm$ 1.91\% &  29.61\% $\pm$ 1.35\% &  29.60\% $\pm$ 1.35\% & \textbf{29.62\% $\pm$ 1.39\%} & 28.63\% $\pm$ 1.66\%\\
 160 & 28.30\% $\pm$ 1.38\% & 30.48\% $\pm$ 1.17\% & 34.86\% $\pm$ 1.12\% & 34.87\% $\pm$ 1.12\% & \textbf{35.02\% $\pm$ 1.11\%} & 33.54\% $\pm$ 1.24\%  \\
 320 & 33.15\% $\pm$ 1.20\% & 36.57\% $\pm$ 0.88\% & 40.46\% $\pm$ 0.73\% & 40.47\% $\pm$ 0.73\% & \textbf{40.66\% $\pm$ 0.72\%} & 39.34\% $\pm$ 0.72\% \\
 640 & 41.66\% $\pm$ 1.09\% & 42.63\% $\pm$ 0.68\% &45.68\% $\pm$ 0.71\% & 45.68\% $\pm$ 0.72\% & \textbf{46.17\% $\pm$ 0.68\%} & 44.91\% $\pm$ 0.72\%   \\
1280 & 49.14\% $\pm$ 1.31\% & 48.86\% $\pm$ 0.68\% & 50.30\% $\pm$ 0.57\% & 50.32\% $\pm$ 0.56\% & \textbf{51.05\% $\pm$ 0.54\%} & 49.74\% $\pm$ 0.42\%\\
\bottomrule
\end{tabular}
}
\end{table*}

%% file: big_dataset_cifar.tex
\begin{table*}[h]
    \caption{The table below shows the average test accuracy and standard deviation of ResNet-34 and FABR on the subsampled and full CIFAR-10 dataset. The test accuracy is average over five independent runs.}
    \centering
    \label{table:big_dataset_cifar}

\scalebox{0.6}{
\begin{tabular}{rccccc}
\toprule
    n &          ResNet-34 &                z=1 &                z=100 &            z=10000 &           z=100000 \\
\midrule
 2560 & 48.12\% $\pm$ 0.69\% & 52.24\% $\pm$ 0.29\% & 52.45\% $\pm$ 0.21\% & \textbf{54.29\% $\pm$ 0.44\%} & 48.28\% $\pm$ 0.37\% \\
 5120 & 56.03\% $\pm$ 0.82\% & 55.34\% $\pm$ 0.32\% & 55.74\% $\pm$ 0.34\% &  \textbf{58.29\% $\pm$ 0.20\%} & 52.06\% $\pm$ 0.08\% \\
10240 & \textbf{63.21\% $\pm$ 0.26\%} & 58.36\% $\pm$ 0.45\% & 58.86\% $\pm$ 0.54\% & 62.17\% $\pm$ 0.35\% & 55.75\% $\pm$ 0.18\% \\
20480 & \textbf{69.24\% $\pm$ 0.47\%} & 61.08\% $\pm$ 0.17\% & 61.65\% $\pm$ 0.27\% & 65.12\% $\pm$ 0.19\% & 59.34\% $\pm$ 0.14\% \\
50000 & \textbf{75.34\% $\pm$ 0.21\%} &  66.38\% $\pm$ 0.00\% &  66.98\% $\pm$ 0.00\% &  68.62\% $\pm$ 0.00\% &  63.25\% $\pm$ 0.00\% \\
\bottomrule
\end{tabular}
}
\end{table*}

%% file: main.bbl
\begin{thebibliography}{49}
\providecommand{\natexlab}[1]{#1}
\providecommand{\url}[1]{\texttt{#1}}
\expandafter\ifx\csname urlstyle\endcsname\relax
  \providecommand{\doi}[1]{doi: #1}\else
  \providecommand{\doi}{doi: \begingroup \urlstyle{rm}\Url}\fi

\bibitem[Ali et~al.(2019)Ali, Kolter, and Tibshirani]{ali2019continuous}
Alnur Ali, J~Zico Kolter, and Ryan~J Tibshirani.
\newblock A continuous-time view of early stopping for least squares
  regression.
\newblock In \emph{The 22nd International Conference on Artificial Intelligence
  and Statistics}, pages 1370--1378. PMLR, 2019.

\bibitem[Allen-Zhu et~al.(2019)Allen-Zhu, Li, and Song]{allen2019convergence}
Zeyuan Allen-Zhu, Yuanzhi Li, and Zhao Song.
\newblock A convergence theory for deep learning via over-parameterization.
\newblock In \emph{International Conference on Machine Learning}, pages
  242--252. PMLR, 2019.

\bibitem[Arora et~al.(2019{\natexlab{a}})Arora, Du, Hu, Li, Salakhutdinov, and
  Wang]{arora2019exact}
Sanjeev Arora, Simon~S Du, Wei Hu, Zhiyuan Li, Russ~R Salakhutdinov, and
  Ruosong Wang.
\newblock On exact computation with an infinitely wide neural net.
\newblock \emph{Advances in Neural Information Processing Systems}, 32,
  2019{\natexlab{a}}.

\bibitem[Arora et~al.(2019{\natexlab{b}})Arora, Du, Li, Salakhutdinov, Wang,
  and Yu]{arora2019harnessing}
Sanjeev Arora, Simon~S Du, Zhiyuan Li, Ruslan Salakhutdinov, Ruosong Wang, and
  Dingli Yu.
\newblock Harnessing the power of infinitely wide deep nets on small-data
  tasks.
\newblock \emph{arXiv preprint arXiv:1910.01663}, 2019{\natexlab{b}}.

\bibitem[Avron et~al.(2017)Avron, Clarkson, and Woodruff]{avron2017faster}
Haim Avron, Kenneth~L Clarkson, and David~P Woodruff.
\newblock Faster kernel ridge regression using sketching and preconditioning.
\newblock \emph{SIAM Journal on Matrix Analysis and Applications}, 38\penalty0
  (4):\penalty0 1116--1138, 2017.

\bibitem[Bartlett et~al.(2020)Bartlett, Long, Lugosi, and
  Tsigler]{bartlett2020benign}
Peter~L Bartlett, Philip~M Long, G{\'a}bor Lugosi, and Alexander Tsigler.
\newblock Benign overfitting in linear regression.
\newblock \emph{Proceedings of the National Academy of Sciences}, 117\penalty0
  (48):\penalty0 30063--30070, 2020.

\bibitem[Belkin(2021)]{belkin2021fit}
Mikhail Belkin.
\newblock Fit without fear: remarkable mathematical phenomena of deep learning
  through the prism of interpolation.
\newblock \emph{Acta Numerica}, 30:\penalty0 203--248, 2021.

\bibitem[Belkin et~al.(2018)Belkin, Ma, and Mandal]{belkin2018understand}
Mikhail Belkin, Siyuan Ma, and Soumik Mandal.
\newblock To understand deep learning we need to understand kernel learning.
\newblock In \emph{International Conference on Machine Learning}, pages
  541--549. PMLR, 2018.

\bibitem[Belkin et~al.(2019{\natexlab{a}})Belkin, Hsu, Ma, and
  Mandal]{belkin2018reconciling}
Mikhail Belkin, Daniel Hsu, Siyuan Ma, and Soumik Mandal.
\newblock Reconciling modern machine-learning practice and the classical
  bias--variance trade-off.
\newblock \emph{Proceedings of the National Academy of Sciences}, 116\penalty0
  (32):\penalty0 15849--15854, 2019{\natexlab{a}}.

\bibitem[Belkin et~al.(2019{\natexlab{b}})Belkin, Rakhlin, and
  Tsybakov]{belkin2019does}
Mikhail Belkin, Alexander Rakhlin, and Alexandre~B Tsybakov.
\newblock Does data interpolation contradict statistical optimality?
\newblock In \emph{The 22nd International Conference on Artificial Intelligence
  and Statistics}, pages 1611--1619. PMLR, 2019{\natexlab{b}}.

\bibitem[Belkin et~al.(2020)Belkin, Hsu, and Xu]{belkin2020two}
Mikhail Belkin, Daniel Hsu, and Ji~Xu.
\newblock Two models of double descent for weak features.
\newblock \emph{SIAM Journal on Mathematics of Data Science}, 2\penalty0
  (4):\penalty0 1167--1180, 2020.

\bibitem[Chizat et~al.(2019)Chizat, Oyallon, and Bach]{chizat2019lazy}
Lenaic Chizat, Edouard Oyallon, and Francis Bach.
\newblock On lazy training in differentiable programming.
\newblock \emph{Advances in Neural Information Processing Systems}, 32, 2019.

\bibitem[Cho and Saul(2009)]{cho2009kernel}
Youngmin Cho and Lawrence Saul.
\newblock Kernel methods for deep learning.
\newblock \emph{Advances in neural information processing systems}, 22, 2009.

\bibitem[Dai et~al.(2014)Dai, Xie, He, Liang, Raj, Balcan, and
  Song]{dai2014scalable}
Bo~Dai, Bo~Xie, Niao He, Yingyu Liang, Anant Raj, Maria-Florina~F Balcan, and
  Le~Song.
\newblock Scalable kernel methods via doubly stochastic gradients.
\newblock \emph{Advances in neural information processing systems}, 27, 2014.

\bibitem[Daniely(2017)]{daniely2017sgd}
Amit Daniely.
\newblock Sgd learns the conjugate kernel class of the network.
\newblock \emph{Advances in Neural Information Processing Systems}, 30, 2017.

\bibitem[Daniely et~al.(2016)Daniely, Frostig, and Singer]{daniely2016toward}
Amit Daniely, Roy Frostig, and Yoram Singer.
\newblock Toward deeper understanding of neural networks: The power of
  initialization and a dual view on expressivity.
\newblock \emph{Advances in neural information processing systems}, 29, 2016.

\bibitem[Du et~al.(2019{\natexlab{a}})Du, Lee, Li, Wang, and
  Zhai]{du2019gradient}
Simon Du, Jason Lee, Haochuan Li, Liwei Wang, and Xiyu Zhai.
\newblock Gradient descent finds global minima of deep neural networks.
\newblock In \emph{International conference on machine learning}, pages
  1675--1685. PMLR, 2019{\natexlab{a}}.

\bibitem[Du et~al.(2018)Du, Zhai, Poczos, and Singh]{du2018gradient}
Simon~S Du, Xiyu Zhai, Barnabas Poczos, and Aarti Singh.
\newblock Gradient descent provably optimizes over-parameterized neural
  networks.
\newblock \emph{arXiv preprint arXiv:1810.02054}, 2018.

\bibitem[Du et~al.(2019{\natexlab{b}})Du, Hou, Salakhutdinov, Poczos, Wang, and
  Xu]{du2019graph}
Simon~S Du, Kangcheng Hou, Russ~R Salakhutdinov, Barnabas Poczos, Ruosong Wang,
  and Keyulu Xu.
\newblock Graph neural tangent kernel: Fusing graph neural networks with graph
  kernels.
\newblock \emph{Advances in neural information processing systems}, 32,
  2019{\natexlab{b}}.

\bibitem[Fern{\'a}ndez-Delgado et~al.(2014)Fern{\'a}ndez-Delgado, Cernadas,
  Barro, and Amorim]{fernandez2014we}
Manuel Fern{\'a}ndez-Delgado, Eva Cernadas, Sen{\'e}n Barro, and Dinani Amorim.
\newblock Do we need hundreds of classifiers to solve real world classification
  problems?
\newblock \emph{The journal of machine learning research}, 15\penalty0
  (1):\penalty0 3133--3181, 2014.

\bibitem[Garriga-Alonso et~al.(2018)Garriga-Alonso, Rasmussen, and
  Aitchison]{garriga2018deep}
Adri{\`a} Garriga-Alonso, Carl~Edward Rasmussen, and Laurence Aitchison.
\newblock Deep convolutional networks as shallow gaussian processes.
\newblock In \emph{International Conference on Learning Representations}, 2018.

\bibitem[Hastie et~al.(2019)Hastie, Montanari, Rosset, and
  Tibshirani]{hastie2019surprises}
Trevor Hastie, Andrea Montanari, Saharon Rosset, and Ryan~J Tibshirani.
\newblock Surprises in high-dimensional ridgeless least squares interpolation.
\newblock \emph{arXiv preprint arXiv:1903.08560}, 2019.

\bibitem[Hazan and Jaakkola(2015)]{hazan2015steps}
Tamir Hazan and Tommi Jaakkola.
\newblock Steps toward deep kernel methods from infinite neural networks.
\newblock \emph{arXiv preprint arXiv:1508.05133}, 2015.

\bibitem[He et~al.(2015)He, Zhang, Ren, and Sun]{he2015delving}
Kaiming He, Xiangyu Zhang, Shaoqing Ren, and Jian Sun.
\newblock Delving deep into rectifiers: Surpassing human-level performance on
  imagenet classification.
\newblock In \emph{Proceedings of the IEEE international conference on computer
  vision}, pages 1026--1034, 2015.

\bibitem[Ioffe and Szegedy(2015)]{ioffe2015batch}
Sergey Ioffe and Christian Szegedy.
\newblock Batch normalization: Accelerating deep network training by reducing
  internal covariate shift.
\newblock In \emph{International conference on machine learning}, pages
  448--456. PMLR, 2015.

\bibitem[Jacot et~al.(2018)Jacot, Gabriel, and Hongler]{jacot2018neural}
Arthur Jacot, Franck Gabriel, and Cl{\'e}ment Hongler.
\newblock Neural tangent kernel: Convergence and generalization in neural
  networks.
\newblock \emph{Advances in neural information processing systems}, 31, 2018.

\bibitem[Kelly et~al.(2022)Kelly, Malamud, and Zhou]{kelly2022virtue}
Bryan~T Kelly, Semyon Malamud, and Kangying Zhou.
\newblock The virtue of complexity in return prediction.
\newblock 2022.

\bibitem[Krizhevsky et~al.(2009)Krizhevsky, Hinton,
  et~al.]{krizhevsky2009learning}
Alex Krizhevsky, Geoffrey Hinton, et~al.
\newblock Learning multiple layers of features from tiny images.
\newblock 2009.

\bibitem[Le~Roux and Bengio(2007)]{le2007continuous}
Nicolas Le~Roux and Yoshua Bengio.
\newblock Continuous neural networks.
\newblock In \emph{Artificial Intelligence and Statistics}, pages 404--411.
  PMLR, 2007.

\bibitem[Lee et~al.(2018)Lee, Bahri, Novak, Schoenholz, Pennington, and
  Sohl-Dickstein]{lee2018deep}
Jaehoon Lee, Yasaman Bahri, Roman Novak, Samuel~S Schoenholz, Jeffrey
  Pennington, and Jascha Sohl-Dickstein.
\newblock Deep neural networks as gaussian processes.
\newblock In \emph{International Conference on Learning Representations}, 2018.

\bibitem[Lee et~al.(2019)Lee, Xiao, Schoenholz, Bahri, Novak, Sohl-Dickstein,
  and Pennington]{lee2019wide}
Jaehoon Lee, Lechao Xiao, Samuel Schoenholz, Yasaman Bahri, Roman Novak, Jascha
  Sohl-Dickstein, and Jeffrey Pennington.
\newblock Wide neural networks of any depth evolve as linear models under
  gradient descent.
\newblock \emph{Advances in neural information processing systems}, 32, 2019.

\bibitem[Lee et~al.(2020)Lee, Schoenholz, Pennington, Adlam, Xiao, Novak, and
  Sohl-Dickstein]{lee2020finite}
Jaehoon Lee, Samuel Schoenholz, Jeffrey Pennington, Ben Adlam, Lechao Xiao,
  Roman Novak, and Jascha Sohl-Dickstein.
\newblock Finite versus infinite neural networks: an empirical study.
\newblock \emph{Advances in Neural Information Processing Systems},
  33:\penalty0 15156--15172, 2020.

\bibitem[LeJeune et~al.(2020)LeJeune, Javadi, and
  Baraniuk]{lejeune2020implicit}
Daniel LeJeune, Hamid Javadi, and Richard Baraniuk.
\newblock The implicit regularization of ordinary least squares ensembles.
\newblock In \emph{International Conference on Artificial Intelligence and
  Statistics}, pages 3525--3535. PMLR, 2020.

\bibitem[Li et~al.(2019)Li, Wang, Yu, Du, Hu, Salakhutdinov, and
  Arora]{li2019enhanced}
Zhiyuan Li, Ruosong Wang, Dingli Yu, Simon~S Du, Wei Hu, Ruslan Salakhutdinov,
  and Sanjeev Arora.
\newblock Enhanced convolutional neural tangent kernels.
\newblock \emph{arXiv preprint arXiv:1911.00809}, 2019.

\bibitem[Lin et~al.(2013)Lin, Chen, and Yan]{lin2013network}
Min Lin, Qiang Chen, and Shuicheng Yan.
\newblock Network in network.
\newblock \emph{arXiv preprint arXiv:1312.4400}, 2013.

\bibitem[Ma and Belkin(2017)]{ma2017diving}
Siyuan Ma and Mikhail Belkin.
\newblock Diving into the shallows: a computational perspective on large-scale
  shallow learning.
\newblock \emph{Advances in neural information processing systems}, 30, 2017.

\bibitem[Matthews et~al.(2018)Matthews, Rowland, Hron, Turner, and
  Ghahramani]{matthews2018gaussian}
Alexander G de~G Matthews, Mark Rowland, Jiri Hron, Richard~E Turner, and
  Zoubin Ghahramani.
\newblock Gaussian process behaviour in wide deep neural networks.
\newblock \emph{arXiv preprint arXiv:1804.11271}, 2018.

\bibitem[Nakkiran et~al.(2021)Nakkiran, Kaplun, Bansal, Yang, Barak, and
  Sutskever]{nakkiran2021deep}
Preetum Nakkiran, Gal Kaplun, Yamini Bansal, Tristan Yang, Boaz Barak, and Ilya
  Sutskever.
\newblock Deep double descent: Where bigger models and more data hurt.
\newblock \emph{Journal of Statistical Mechanics: Theory and Experiment},
  2021\penalty0 (12):\penalty0 124003, 2021.

\bibitem[Neal(1996)]{neal1996priors}
Radford~M Neal.
\newblock Priors for infinite networks.
\newblock In \emph{Bayesian Learning for Neural Networks}, pages 29--53.
  Springer, 1996.

\bibitem[Novak et~al.(2018)Novak, Xiao, Bahri, Lee, Yang, Hron, Abolafia,
  Pennington, and Sohl-dickstein]{novak2018bayesian}
Roman Novak, Lechao Xiao, Yasaman Bahri, Jaehoon Lee, Greg Yang, Jiri Hron,
  Daniel~A Abolafia, Jeffrey Pennington, and Jascha Sohl-dickstein.
\newblock Bayesian deep convolutional networks with many channels are gaussian
  processes.
\newblock In \emph{International Conference on Learning Representations}, 2018.

\bibitem[Novak et~al.(2019)Novak, Xiao, Hron, Lee, Alemi, Sohl-Dickstein, and
  Schoenholz]{novak2019neural}
Roman Novak, Lechao Xiao, Jiri Hron, Jaehoon Lee, Alexander~A Alemi, Jascha
  Sohl-Dickstein, and Samuel~S Schoenholz.
\newblock Neural tangents: Fast and easy infinite neural networks in python.
\newblock \emph{arXiv preprint arXiv:1912.02803}, 2019.

\bibitem[Olson et~al.(2018)Olson, Wyner, and Berk]{olson2018modern}
Matthew Olson, Abraham Wyner, and Richard Berk.
\newblock Modern neural networks generalize on small data sets.
\newblock \emph{Advances in Neural Information Processing Systems}, 31, 2018.

\bibitem[Rahimi and Recht(2007)]{rahimi2007random}
Ali Rahimi and Benjamin Recht.
\newblock Random features for large-scale kernel machines.
\newblock \emph{Advances in neural information processing systems}, 20, 2007.

\bibitem[Shankar et~al.(2020)Shankar, Fang, Guo, Fridovich-Keil, Ragan-Kelley,
  Schmidt, and Recht]{shankar2020neural}
Vaishaal Shankar, Alex Fang, Wenshuo Guo, Sara Fridovich-Keil, Jonathan
  Ragan-Kelley, Ludwig Schmidt, and Benjamin Recht.
\newblock Neural kernels without tangents.
\newblock In \emph{International Conference on Machine Learning}, pages
  8614--8623. PMLR, 2020.

\bibitem[Spigler et~al.(2019)Spigler, Geiger, d’Ascoli, Sagun, Biroli, and
  Wyart]{spigler2019jamming}
Stefano Spigler, Mario Geiger, St{\'e}phane d’Ascoli, Levent Sagun, Giulio
  Biroli, and Matthieu Wyart.
\newblock A jamming transition from under-to over-parametrization affects
  generalization in deep learning.
\newblock \emph{Journal of Physics A: Mathematical and Theoretical},
  52\penalty0 (47):\penalty0 474001, 2019.

\bibitem[Tsigler and Bartlett(2020)]{tsigler2020benign}
A.~Tsigler and P.~L. Bartlett.
\newblock Benign overfitting in ridge regression, 2020.

\bibitem[Williams(1997)]{williams1997computing}
Christopher~KI Williams.
\newblock Computing with infinite networks.
\newblock In \emph{Advances in Neural Information Processing Systems 9:
  Proceedings of the 1996 Conference}, volume~9, page 295. MIT Press, 1997.

\bibitem[Zandieh et~al.(2021)Zandieh, Han, Avron, Shoham, Kim, and
  Shin]{NEURIPS2021_08ae6a26}
Amir Zandieh, Insu Han, Haim Avron, Neta Shoham, Chaewon Kim, and Jinwoo Shin.
\newblock Scaling neural tangent kernels via sketching and random features.
\newblock In M.~Ranzato, A.~Beygelzimer, Y.~Dauphin, P.S. Liang, and J.~Wortman
  Vaughan, editors, \emph{Advances in Neural Information Processing Systems},
  volume~34, pages 1062--1073. Curran Associates, Inc., 2021.
\newblock URL
  \url{https://proceedings.neurips.cc/paper/2021/file/08ae6a26b7cb089ea588e94aed36bd15-Paper.pdf}.

\bibitem[Zhang et~al.(2016)Zhang, Bengio, Hardt, Recht, and
  Vinyals]{zhang2016understanding}
Chiyuan Zhang, Samy Bengio, Moritz Hardt, Benjamin Recht, and Oriol Vinyals.
\newblock Understanding deep learning requires rethinking generalization.
\newblock \emph{arXiv preprint arXiv:1611.03530}, 2016.

\end{thebibliography}
